\documentclass{article}

\usepackage{fullpage}
\usepackage{charter, times}
\date{}

\usepackage{hyperref}       
\usepackage{url}            
\usepackage{booktabs}       
\usepackage{amsfonts}       
\usepackage{nicefrac}       
\usepackage{microtype}      
\usepackage{graphicx}
\usepackage{balance}  
\usepackage{soul}
\usepackage{bbm}
\usepackage{url}
\usepackage[utf8]{inputenc}
\usepackage[small]{caption}
\usepackage{graphicx}
\usepackage{amsmath}
\usepackage{amssymb}
\usepackage{amsthm}
\usepackage{booktabs}
\usepackage{algorithm}
\usepackage{algorithmic}
\usepackage[dvipsnames]{xcolor}

\usepackage{subfigure}
\usepackage{multirow}
\usepackage{bm}
\usepackage{amssymb}
\usepackage{fullpage}
\usepackage{makecell}
\usepackage{authblk}
\usepackage{lipsum}

\usepackage{natbib}
\renewcommand{\bibname}{References}
\renewcommand{\bibsection}{\subsubsection*{\bibname}}

\usepackage{amsmath,amsfonts,bm}

\def\eqref#1{equation~\ref{#1}}

\def\floor#1{\lfloor #1 \rfloor}
\def\1{\bm{1}}

\DeclareMathAlphabet{\mathsfit}{\encodingdefault}{\sfdefault}{m}{sl}
\SetMathAlphabet{\mathsfit}{bold}{\encodingdefault}{\sfdefault}{bx}{n}

\usepackage{wrapfig}
\usepackage{hyperref}
\usepackage{url}
\usepackage{graphicx}
\usepackage{subfigure}
\usepackage{enumitem}
\usepackage{algorithm}
\usepackage{amsthm}
\usepackage{float}
\usepackage{footnote}
\usepackage[textsize=small]{todonotes}

\DeclareMathOperator{\vectorized}{vec}

\newtheorem{lemma}{Lemma}
\newtheorem{corollary}{Corollary}
\newtheorem{theorem}{Theorem}

\usepackage{xspace}
\newcommand{\name}{{\bf DeGNN}\xspace}

\title{
\name:
{\bf Characterizing
and Improving Graph Neural Networks with Graph Decomposition}}

\author{%
   \text{Xupeng Miao}${}^{\dagger *}$ \hspace{0.3em} Nezihe Merve G\"urel${}^{\ddagger}$\thanks{equal contribution} \hspace{0.3em} \text{Wentao Zhang}${}^{\dagger *}$ \hspace{0.3em} \text{Zhichao Han}${}^{\dagger\dagger}$ \hspace{0.3em} \text{Bo Li}${}^{\ddagger\ddagger}$\hspace{0.3em} \text{Wei Min}${}^{\dagger\dagger}$\\ 
  \vspace{-0.7em}
  \text{Xi Rao}${}^{\ddagger}$ \hspace{0.3em} \text{Hansheng Ren}${}^{\mathsection}$ \hspace{0.3em} \text{Yinan Shan}${}^{\dagger\dagger}$ \hspace{0.3em} \text{Yingxia Shao}${}^{\mathparagraph}$ \hspace{0.3em} \text{Yujie Wang}${}^{\dagger}$\hspace{0.3em} \text{Fan Wu}${}^{\ddagger\ddagger}$ \hspace{0.3em} \text{Hui Xue}${}^{\mathsection}$\\ 
   \vspace{0.1em}
  \text{Yaming Yang}${}^{\mathsection}$ \hspace{0.3em} \text{Zitao Zhang}${}^{\dagger\dagger}$ \hspace{0.3em}\text{Yang Zhao}${}^{\dagger\dagger}$ \hspace{0.3em} \text{Shuai Zhang}${}^{\ddagger}$ \hspace{0.3em}\text{Yujing Wang}${}^{\mathsection}$ \hspace{0.3em}\text{Bin Cui}${}^{\dagger}$ \hspace{0.3em}\text{Ce Zhang}${}^{\ddagger}$\\
  \vspace{0.7em}
  ${}^{\dagger}$ \text{Peking University}\hspace{0.3em} ${}^{\ddagger}$\text{ETH Z\"urich}\hspace{0.3em} ${}^{\dagger\dagger}$\text{eBay} \hspace{0.3em} ${}^{\ddagger\ddagger}$\text{University of Illinois at Urbana-Champaign} \\
  \vspace{0.1em}
  ${}^{\mathsection}$\text{Microsoft Research Asia}\hspace{0.3em}${}^{\mathparagraph}$\text{Beijing University of Posts and Telecommunications}
}

\begin{document}
\maketitle
\begin{abstract} 

Despite the wide application of Graph Convolutional Network (GCN), one major limitation is that it does not benefit from the increasing depth and suffers from the oversmoothing problem. In this work, we first characterize this phenomenon from the information-theoretic perspective and show that under certain conditions, the mutual information between the output after $l$ layers and the input of GCN converges to 0 exponentially with respect to $l$. 
We also show that, on the other hand, graph decomposition can potentially weaken the condition of such convergence rate, which enabled our analysis for GraphCNN.
While different graph structures can only benefit from the corresponding decomposition, in practice, we propose an automatic connectivity-aware graph decomposition algorithm, \name, to improve the performance of general graph neural networks.
Extensive experiments on widely adopted benchmark datasets demonstrate that
often \name can not only significantly boost the performance of corresponding GNNs, but also achieves the state-of-the-art performances.

\end{abstract}

\vspace{-0.5em}
\section{Introduction}
\vspace{-0.5em}
Graph Convolutional Network (GCN)~\citep{DBLP:conf/iclr/KipfW17} has attracted intensive interests recently. The GCNs pave a new way to effectively learn representations for graph-structured data and have a wide spectrum of applications including semi-supervised node classification~\citep{DBLP:conf/iclr/KipfW17}, link prediction~\citep{berg2017graph}, recommendation systems~\citep{ying2018graph}, chemical compounds analysis~\citep{such2017robust}, transportation systems~\citep{li2017graph}, etc. Despite its success, one limitation of GCN is that it suffers from performance degradation when it goes deeper. This phenomenon is also identified as the {\em oversmoothing} problem~\cite{DeeperInsight, Asymptotic}: when multiple GCN layers are stacked together, the output will converge to a region that is \textit{independent} of weights and 
inputs, thus
degrades the quality
significantly with respect to the depth.
Integrating 
techniques such 
as 
residual connections (ResGCN) 
and dense
connections
(DenseGCN)
can help accommodate this problem to a certain extend;
however,
this limitation 
remains~\citep{DBLP:conf/iclr/KipfW17}.


It is
also known that
partitioning the graph with a hand-picked structure 
can help a range of
tasks.
For example, 
thinking of an image as a graph, if we
decompose it into multiple subgraphs
(as illustrated 
in
Figure~\ref{fig:gcn}), it is possible
to design a GCN-variant to 
implement 
a standard CNN-like 
model, which
obviously benefits
from going deeper.
GraphCNN~\cite{such2017robust} 
is one such example
of taking advantage of
graph decomposition. However, this 
requires us to know the ``\textit{right}''
decomposition of a graph, which is often
not available in practice.

In this paper, we are inspired by these
observations and results, and ask two questions:
\begin{enumerate}
\item From the \underline{\textit{theoretical}} perspective,
can we explain the significant impact of
graph decomposition on the performance of 
Graph Neural Networks? 
\item From the \underline{\textit{empirical}} perspective,
can we \textit{automatically} decompose a graph 
and improve the 
quality of state-of-the-art Graph Neural Networks?
\end{enumerate}

\begin{figure*}[t]
    \begin{center}
    \includegraphics[width=0.8\textwidth]{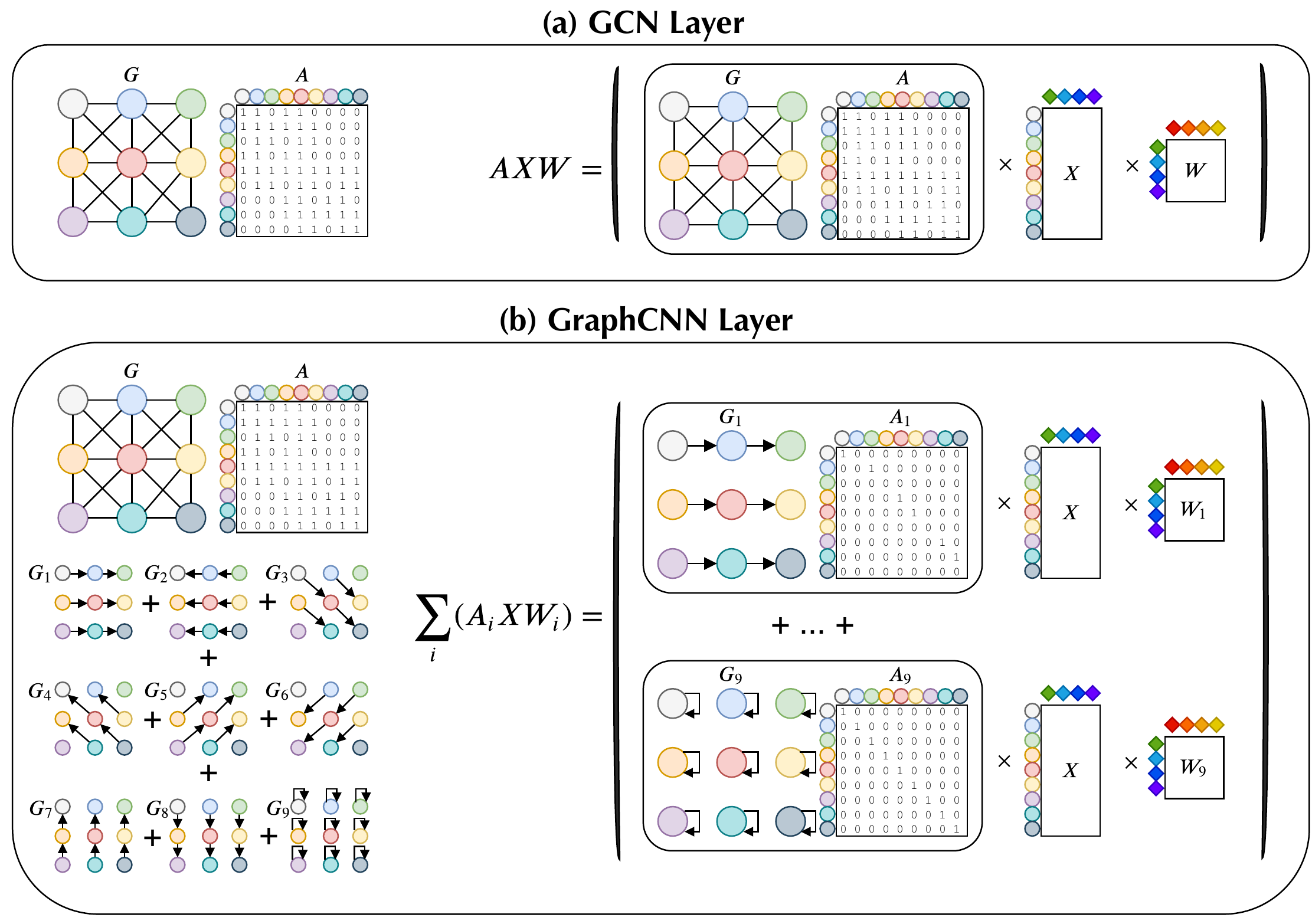}
    \end{center}
    \caption{Illustration of one layer in GCN
    and one layer under one decomposition strategy
    in GraphCNN. ${\bf A}$ is the adjacency matrix,
    ${\bf X}$ is the input, and ${\bf W}$ (${\bf W}_i$)
    are learnable weights. In GraphCNN,
    ${\bf A} = \sum_i {\bf A}_i$ and ${\bf A}_i \cap {\bf A}_j = \emptyset$
    for $i \ne j$.
    In our experiments and analysis, 
    we follow the original normalized ${\bf A}$ in GCN~\cite{DBLP:conf/iclr/KipfW17}.
    }
	\label{fig:gcn}
\end{figure*}

\underline{Our first contribution} is
to take the first step towards the 
theoretical analysis on the impact 
of graph decomposition. We take
an information theoretical view and 
analyze the infinite-sample behaviour of Shannon's {\em mutual information} between the output after $l$ layers and the input,
$\mathcal{I}({\bf x}; {\bf y}^{(l)})$. 
When $\mathcal{I}({\bf x}; {\bf y}^{(l)}) = \mathcal{H}({\bf x})$,  it indicates
that all information
in the input are \textit{fully preserved}
after $l$ layers;
whereas when  $\mathcal{I}({\bf x}; {\bf y}^{(l)}) = 0$, it indicates that \textit{all information
are lost}.
We show that:
\begin{enumerate}
\item (Theorem 1, 2) Under certain conditions (on the singular value
of the graph),  mutual information 
$\mathcal{I}({\bf x}; {\bf y}^{(l)})$
for GCN
converges to 0 exponentially fast with
respect to the depth $l$, corresponding to
the  
oversmoothing problem of GCN
in practice;
\item (Theorem 3, 4) Only under a \textit{much weaker} condition,
the mutual information 
$\mathcal{I}({\bf x}; {\bf y}^{(l)})$
of GraphCNN with decomposition
converges to 0.
\end{enumerate}

The theoretical analysis is non-trivial --- 
in a concurrent work~\cite{Asymptotic}, the authors conducted engaged analysis, from dynamic system perspective, and 
lead to a similar result for GCN (Theorem 1, 2). Our information theoretical perspective not
only provides a much simpler, but equally tight analysis for GCN, but
more importantly, our analysis makes it possible to analyze 
more complex cases for GraphCNN with the presence of decomposition (Theorem 3, 4).

Given the theoretical analysis, one question lingers ---
\textit{can we design practical algorithms to take
advantage of graph decomposition?}
The design of
the decomposition strategy is a
delicate matter.
\underline{Our second contribution}
is a novel
graph connectivity aware decomposition algorithm
to automatically
decompose a graph into multiple subgraphs
and use them to improve the quality of 
Graph Neural Networks. 


We conduct extensive experiments by applying
our decomposition method to GCN~\cite{DBLP:conf/iclr/KipfW17},
JK-Net~\cite{DBLP:conf/icml/XuLTSKJ18}, ResGCN~\cite{DBLP:conf/iclr/KipfW17}, and DenseGCN~\cite{DeeperInsight} --- On all 
these architectures, we show that
our decomposition methods
provide significant improvement. We
then compare
our methods with a range of state-of-the-art
models including GPNN~\cite{DBLP:conf/iclr/LiaoBTGUZ18},
NGCN~\cite{DBLP:conf/uai/Abu-El-HaijaKPL19}, DGCN~\cite{zhuang2018dual}, 
DropEdge~\cite{rong2019dropedge}, LGCN~\cite{gao2018large},
GMI~\cite{DBLP:conf/www/PengHLZRXH20}, and GAT~\cite{velivckovic2017graph}.
We show that with our graph decomposition 
method, simpler models such as 
DenseGCN can often outperform the 
best among these state-of-the-art
models on 12 datasets.

\vspace{-0.5em}
\section{Related Work}
\vspace{-0.5em}

GCN and its variants have achieved promising results on various graph applications, while one limitation of GCN is that its performance would not improve with the increase of network depths. For instance, \cite{DBLP:conf/iclr/KipfW17} show that a two-layer GCN would achieve the best performance on a classic graph dataset while stacking more layers cannot help to improve the performance. Several studies have been conducted \citep{Review, Comprehensive} trying to figure out the reasons behind the depth limitation and provide workarounds. \cite{SGC} hypothesizes that nonlinearity between GCN layers is not critical, which essentially implies that the deep GCN model lacks sufficient expressive ability since it is a linear model. DropEdge~\cite{rong2019dropedge} aims to address the oversmoothing problem by randomly removing some edges from the graph. There is also a rising interest in deepening GCN by utilizing some techniques that are used to build deeper CNN architectures (e.g., ResGCN~\cite{DBLP:conf/iclr/KipfW17}, DenseGCN~\cite{DeepGCN}, JK-Net~\cite{DBLP:conf/icml/XuLTSKJ18}).
However, these lacks of evidence showing whether these techniques are helpful to improve the performance of general GNNs.

To further understand this phenomenon in GCN, 
\cite{DeeperInsight} shows that GCN is a special form of Laplacian smoothing, and they prove that, under certain conditions, by repeatedly applying Laplacian smoothing many times, the features of vertices within each connected component of the graph will converge to the same value. Therefore, the oversmoothing property of GCN will make the features indistinguishable and thus hurt the classification accuracy. \cite{Asymptotic} conducts more engaged theoretical analysis.
The goal of this work is to go beyond 
the analysis of oversmoothing, instead,
we to analyze how graph decomposition
can help 
and propose practical algorithms inspired
by our analysis.

In addition, GMI~\cite{DBLP:conf/www/PengHLZRXH20} proposes
to maximize the correlation between input graphs and high-level hidden representations; and improves the performance on both transductive and inductive tasks.
Compared with these work, we aim to develop the theoretic analysis to explain the information loss in GNNs directly from the information theoretic perspective. In addition, we aim to theoretically show that the decomposition in GraphCNN can help to slow down such information loss, which in turn inspires practical graph decomposition algorithm for general graph-structured data.

\vspace{-0.5em}
\section{Information Loss in Graph Neural Networks}\label{Section: Theoretical Analysis}
\vspace{-0.5em}

Let $\mathcal{G}=(\mathcal{V},\mathcal{E})$ be an undirected graph with a vertex set $v_i \in \mathcal{V}$ and edge set $e_{i, j}\in \mathcal{E}$. We refer to ${v}_i$ as a node, and ${\bf x}_i\in \mathbb{R}^d$ associated with ${v}_i$ as its features. We denote the node feature attributes by ${\bf X}\in \mathbb{R}^{n\times d}$ whose rows are given by ${\bf x}_i$. The adjacency matrix ${\bf A}$ (weighted or binary) is derived as an $n\times n$ matrix with $({\bf A})_{i, j}=e_{i, j}$ if $e_{i, j}\in \mathcal{E}$, and $({\bf A})_{i, j}=0$ elsewhere.

We define the following operator $f:\mathbb{R}^n \rightarrow \mathbb{R}^n$ that is composed of (1) a linear function parameterized by the adjacency matrix ${\bf A}$ and a weight matrix ${\bf W}^{(i+1)}$ at layer $i+1$, and (2) an activation function. Given the input matrix ${\bf X}$, let ${\bf Y}^{(0)}= {\bf X}$. Each layer of the graph neural network maps it to an output vector of the same shape: ${\bf Y}^{(i+1)} = \sigma ({\bf A}{\bf Y}^{(i)}{\bf W}^{(i+1)}).$
In GraphCNN~\cite{such2017robust}, the adjacency matrix ${\bf A} \in \mathbb{R}^{n\times n}$ is decomposed into $K$ additive $n\times n$ matrices such that ${\bf A} = \sum_{k=1}^K{\bf A}_k$. The layer-wise propagation rule becomes: ${\bf Y}^{(i+1)} = \sigma(\sum_{k=1}^K {\bf A}_k{\bf Y}^{(i)}{\bf W}_k^{(i+1)}).$

In this paper, we denote the $j$th singular value of a matrix by $\lambda_j(\cdot)$. We further denote the vectorized input ${\bf X}$ and output after the $l$th layer ${\bf Y}^{(l)}$ by ${\bf x}$ and ${\bf y}^{(l)}$, respectively. For $n$-dimensional real random vectors ${\bf x}$ and ${\bf y}$ defined over finite alphabets $\mathcal{X}^n$ and $\Omega^n$,
we denote entropy of {\bf x} by $\mathcal{H}({\bf x})$, and mutual information between {\bf x} and {\bf y} by $\mathcal{I}({\bf x}; {\bf y})$. In the following analysis, we focus on two measures to investigate the effect of decomposition, that is, \emph{information preservation} $\mathcal{I}({\bf x}; {\bf y}^{(l)})$ and the \emph{information loss} $\mathcal{L}({\bf y}^{(l)})=\mathcal{H}({\bf x}|{\bf y}^{(l)})$ (relative entropy of ${\bf x}$ with respect to ${\bf y}^{(l)}$). We measure the information decay in GNNs at different output layers $l$: lower information loss or larger information preservation indicates more meaningful learned features for GNNs in the infinite-sample regime.

\vspace{-0.5em}
\subsection{Information Loss in GCN}\label{Section: GCN}
\vspace{-0.5em}

In this section, our goal is to investigate the regimes where GCN (1) does not benefit from going deeper, or (2) is guaranteed to preserve all information at its output. We aim to understand this by analyzing the behavior of mutual information between the input and the output of certain network layers at different depths. Due to the space limitation, we relegate all the proofs to the Appendix.

First, we formulate the relationship between input and output layers incorporating the
non-linear activation functions. 
In this paper, we focus on the 
most popular choices, i.e., ReLU,
and leave the study of other 
functions to future work.
The characteristics of the layer-wise propagation rule of GCN leads us to the following result:
\begin{lemma}\label{Lemma: Linearize GCN}
Let $\otimes$ denote the Kronecker product. For GCNs with parametric ReLU activations $\sigma: x \rightarrow \max(x, ax)$ with $a\in (0,1)$, we define ${\bf P}^{(i+1)}$ as a diagonal mask matrix whose nonzero entries are in $\{a, 1\}$ such that $({\bf P}^{(i+1)})_{j, j}=1$ if $\big(({\bf W}^{(i+1)}\otimes {\bf A}){\bf y}^{(i)}\big)_j\geq 0$, and $({\bf P}^{(i+1)})_{j, j}=a$ elsewhere. ${\bf y}^{(l)}$ can be written as
{\small
\begin{equation*}
{\bf y}^{(l)} = {\bf P}^{(l)}({\bf W}^{(l)}\otimes {\bf A}) \cdots {\bf P}^{(2)}({\bf W}^{(2)}\otimes {\bf A}) {\bf P}^{(1)}({\bf W}^{(1)}\otimes {\bf A}) {\bf x}.
\end{equation*}}
\end{lemma}

Following our earlier discussion, we will now state our first result which characterizes the regime in which the information propagated across the graph neural network layers exponentially decays to 0.
\begin{theorem}\label{Theorem: GCN information loss} 
Suppose $\sigma_{\bf A} = \max_{j}\lambda_j({\bf A})$ and $\sigma_{\bf W} = \sup_{i\in \mathbb{N}^+} \max_{j} \lambda_j({\bf W}^{(i)})$. If $\sigma_{\bf A} \sigma_{\bf W} <1$, then  $\mathcal{I}({\bf x}; {\bf y}^{(l)})=\mathcal{O}\big((\sigma_{\bf A} \sigma_{\bf W})^l\big)$, and hence $\lim_{l\rightarrow{\infty}}\mathcal{I}({\bf x}; {\bf y}^{(l)})=0$.
\end{theorem}


This shows that under certain conditions the information after $l$ GCN layers with (parametric) ReLUs asymptotically converges to 0 exponentially fast.
Interestingly, there are also regimes in which
GCN will perfectly preserve 
the information, stated 
as follows:
\begin{theorem}\label{Theorem: GCN no information loss} 
Following Theorem~\ref{Theorem: GCN information loss}, let  $\gamma_{\bf A} = \min_{j}\lambda_j({\bf A})$ and $\gamma_{\bf W} = \inf_{i\in \mathbb{N}^+} \min_{j} \lambda_j({\bf W}^{(i)})$. If $a\gamma_{\bf A} \gamma_{\bf W} \geq1$, then $\forall l \in \mathbb{N}^+$ the information loss $\mathcal{L}( {\bf y}^{(l)})= 0$. 
\end{theorem}

\paragraph{Effect of Normalized Laplacian:}
The results obtained above holds for any adjacency matrix ${\bf A}\in \mathbb{R}^{n\times n}$. The unnormalized ${\bf A}$, however, comes with a major drawback as changing the scaling of feature vectors. To overcome this problem, ${\bf A}$ is often normalized such that its rows sum to one. We then adopt our results to GCN with normalized Laplacian whose largest singular value is one, and obtain the following results.
\begin{corollary}\label{Corollary: GCN Laplacian}
Let ${\bf D}$ denote the degree matrix such that $({\bf D})_{j, j}= \sum_m ({\bf A})_{j, m}$, and ${\bf L}$ be the associated normalized Laplacian ${\bf L}={\bf D}^{-1/2}{\bf A}{\bf D}^{-1/2}$. Suppose GCN uses the following mapping ${\bf Y}^{(i+1)}=\sigma({\bf L}{\bf Y}^{(i)}{\bf W}^{(i)})$ and $\sigma_{\bf W} = \sup_{i} \max_{j} \lambda_j({\bf W}^{(i+1)})$. If $\sigma_{\bf W} <1$, then $\mathcal{I}({\bf x}; {\bf y}^{(l)})=\mathcal{O}\big({\sigma^l_{\bf W}}\big)$, and hence $\lim_{l\rightarrow{\infty}}\mathcal{I}({\bf x}; {\bf y}^{(l)})=0$. 
\end{corollary}
%
This indicates that with the standard normalized adjacency matrix, the mutual information between the input and the output of $l$th layer of GCN will decay to 0 exponentially fast.

\vspace{-0.5em}
\subsection{Information Loss in GraphCNN}
\vspace{-0.5em}

Motivated by the graph decomposition strategy adopted by several work including GraphCNN, in this section we aim to analyze the information loss after graph decomposition, and understand whether the information can be preserved by aggregating local sub-graphs.
In particular, we take the GraphCNN as as an example which sums the decomposed graphs together as the adjacency matrix to perform the analysis.

Similarly as in Lemma~\ref{Lemma: Linearize GCN}, 
${\bf y}^{(l)}$ can be reduced to
${\bf y}^{(l)}={\bf P}^{(l)}\sum_{k_{l}=1}^K ({\bf W}_{k_{l}}^{(l)}\otimes {\bf A}_{k_{l}})\cdots ({\bf W}_{k_{2}}^{(2)}\otimes {\bf A}_{k_{2}})({\bf W}_{k_{1}}^{(1)}\otimes {\bf A}_{k_{1}}){\bf x}$ for a diagonal mask matrix ${\bf P}^{(i+1)}$ such that $({\bf P}^{(i+1)})_{j, j}=1$ if $\sum_{k_{i+1}=1}^K ({\bf W}_{k_{i+1}}^{(i+1)}\otimes {\bf A}_{k_{i+1}}){\bf y}^{(i)}\geq 0$, and $({\bf P}^{(i+1)})_{j, j}=a$ otherwise.

Following a similar proof for GCN, we obtain the following result for GraphCNN:
\begin{theorem}\label{Theorem: GCNN Information Loss General}
Let $\sigma^{(i)}$ denotes the maximum singular value of \  ${\bf P}^{(i)}\sum_{k_{i}=1}^K ({\bf W}_{k_{i}}^{(i)}\otimes {\bf A}_{k_{i}})$ such that $\sigma^{(i)} = \max_j \lambda_j\big({\bf P}^{(i)}\sum_{k_{i}} ({\bf W}_{k_{i}}^{(i)}\otimes {\bf A}_{k_{i}})\big)$. If $\sup_{i\in \mathbb{N}^+} \sigma^{(i)} <1$, then $\mathcal{I}({\bf x}; {\bf y}^{(l)})=\mathcal{O}\big((\sup_{i\in \mathbb{N}^+} \sigma^{(i)})^l\big)$, and hence $\lim_{l\rightarrow \infty} \mathcal{I}({\bf x}; {\bf y}^{(l)})=0$.
\end{theorem}

Theorem~\ref{Theorem: GCNN Information Loss General} describes the condition on the layer-wise weight matrices ${\bf W}_k$ where GraphCNN fails in capturing the feature characteristics at its output in the asymptotic regime. We then state the second result for GraphCNN which ensures the information loss $\mathcal{L}({\bf y}^{(l)})=0$ as follows.
\begin{sloppypar}
\begin{theorem}\label{Theorem: GCNN Information No Loss General}
Consider the propagation rule of GraphCNN.
Let $\gamma^{(i)}$ denotes the minimum singular value of \  ${\bf P}^{(i)}\sum_{k_{i}=1}^K ({\bf W}_{k_{i}}^{(i)}\otimes {\bf A}_{k_{i}})$ such that $\gamma^{(i)} = \min_j \lambda_j\big({\bf P}^{(i)}\sum_{k_{i}=1}^K ({\bf W}_{k_{i}}^{(i)}\otimes {\bf A}_{k_{i}})\big)$. If $\inf_i \gamma^{(i)} \geq 1$, then $\forall l\in \mathbb{N}^+$ we have $\mathcal{L}({\bf y}^{(l)})=0$.
\end{theorem}
\end{sloppypar}

{\bf Proof Sketch.} Following Lemma~\ref{Lemma: Linearize GCN}, the key step in proving above theorems is as follows. Consider the singular value decomposition {\small ${\bf U}{\bf \Lambda}{\bf V}^T={\bf P}^{(l)}({\bf W}^{(l)}\otimes {\bf A}) ... {\bf P}^{(2)}({\bf W}^{(2)}\otimes {\bf A}) {\bf P}^{(1)}({\bf W}^{(1)}\otimes {\bf A})$} such that {\small $({\bf \Lambda})_{j,j}=\lambda_j({\bf P}^{(l)}({\bf W}^{(l)}\otimes {\bf A}) ... {\bf P}^{(2)}({\bf W}^{(2)}\otimes {\bf A}) {\bf P}^{(1)}({\bf W}^{(1)}\otimes {\bf A}))$}, and let $\tilde{\bf x}={\bf V}^T{\bf x}$. We have
\begin{equation}
    \begin{split}
        \mathcal{I}({\bf x}; {\bf y}^{(l)}) \stackrel{(1)}{=} \mathcal{I}(\tilde{\bf x}; \Lambda\tilde{\bf x}) &\stackrel{(2)}{\leq}\mathcal{H}(\tilde{\bf x}) \stackrel{(3)}{=}\mathcal{H}({\bf x})
    \end{split}
\end{equation}
where (1, 3) results from that ${\bf U}$ and ${\bf V}$ are invertible, and equality holds in (2) \textit{iff} ${\bf \Lambda}$ is invertible, i.e., singular values of {\small ${\bf P}^{(l)}({\bf W}^{(l)}\otimes {\bf A}) ... {\bf P}^{(2)}({\bf W}^{(2)}\otimes {\bf A}) {\bf P}^{(1)}({\bf W}^{(1)}\otimes {\bf A})$} are nonzero. Theorems~\ref{Theorem: GCN information loss}, \ref{Theorem: GCN no information loss}, \ref{Theorem: GCNN Information Loss General} and \ref{Theorem: GCNN Information No Loss General} can be inferred from Lemma~\ref{Lemma: MI General}. That is, $\mathcal{I}({\bf x}; {\bf y}^{(l)})=0$ \textit{iff} $\max_j ({\bf \Lambda}^l)_{j, j} =0$ in the asymptotic regime. Similarly, \textit{iff} $\min_j ({\bf \Lambda}^l)_{j, j} >0$, $\mathcal{I}({\bf x}; {\bf y}^{(l)})$ is maximized and given by $\mathcal{H}({\bf x})$, hence  $\mathcal{L}({\bf y}^{(l)})=0$ .

In order to understand the role of decomposition in GraphCNN, we revisit the conditions on full information loss ($\mathcal{I}({\bf x}; {\bf y}^{(l)})=0$) and full information preservation ($\mathcal{L}({\bf y}^{(l)})=0$) for a specific choice of decomposition, which will be used to demonstrate the information processing capability.

\begin{corollary}\label{Corollary: GCNN information loss decompA}
Suppose the singular value decomposition of ${\bf A}$ is given by ${\bf A}={\bf U}_{\bf A}{\bf S V}_{\bf A}^T$, and each ${\bf A}_k$ is set to ${\bf A}_k = {\bf U}_{\bf A}{\bf S}_k{\bf V}_{\bf A}^T$ where $({\bf S}_k)_{m, m} = \lambda_m({\bf A})$ if $k=m$ and $({\bf S}_k)_{m, m} = 0$ elsewhere. We then have the following results:
For $\sigma_{{\bf A}_k} = \lambda_k({\bf A})$ and $\sigma_{{\bf W}_k}=\sup_{i\in \mathbb{N}^+}\max_j\lambda_j({\bf W}_k^{(i)})$, i.e., if $\sigma_{{\bf A}_k}\sigma_{{\bf W}_k}<1$ $\underline{\forall k=\{1, 2, \dots, n\}}$, then $\lim_{l \rightarrow \infty} \mathcal{I}({\bf x}; {\bf y}^{(l)})=0$.
\end{corollary}
\begin{corollary}\label{Corollary: GCNN no information no loss decompA}
Let $\gamma_{{\bf W}_k} = \inf_{i\in \mathbb{N}^+} \min_j \lambda_j({\bf W}_k^{(i)})$. If $a\sigma_{{\bf A}_k}\gamma_{{\bf W}_k}\geq 1$, $\underline{\forall k \in \{1, 2, \dots, n\}}$, then $\mathcal{L}({\bf y}^{(l)}) = 0$ $\forall l\in \mathbb{N}^+$.
\end{corollary}


\paragraph*{Discussion: Impact of
Decomposition.}
Consider the setting where ${\bf A}$ is fixed for both GCN and GraphCNN. The discussion below will revolve around the regime of singular values in layer-wise weight matrices, ${\bf W}_{\text{GCN}}^{(i)}$ and ${\bf W}_{\text{GraphCNN}}^{(i)}$  where the information loss $\mathcal{L}( {\bf y}^{(l)})=0$ for specific decomposition strategy used in Corollary~\ref{Corollary: GCNN no information no loss decompA}.

Recall from Theorem~\ref{Theorem: GCN no information loss} and Corollary~\ref{Corollary: GCNN no information no loss decompA} that while GCN requires singular values of all weight matrices ${\bf W}_{\text{GCN}}^{(i)}$ to compensate for the minimum singular value of ${\bf A}$ such that $\min_j\lambda_j({\bf W}_{\text{GCN}}^{(i)})\geq \frac{1}{a\min_k\lambda_k({\bf A})}$ to ensure $\mathcal{L}({\bf y}^{(l)})=0$, GraphCNN relaxes this condition by introducing a milder constraint. That is, the singular values of its weight matrices ${\bf W}_{\text{k, GraphCNN}}^{(i)}$ need to compensate only for the singular value of their respective component ${\bf A}_k$, meaning $\min_j \lambda_j({\bf W}_{\text{k, GraphCNN}}^{(i)})\geq \frac{1}{a\lambda_k({\bf A})}$ implies $\mathcal{L}({\bf y}^{(l})=0$.

The decomposition makes deep GCN training easier by permitting a much larger regime of model weights where the information is still preserved. In other words, under the same weight characteristics (singular values of layer-wise weight matrices), the decomposed GCN will be able to preserve more information of the node features than the vanilla GCN when going deeper.
So far, we theoretically justify the potential of graph decomposition in the infinite-sample regime. For the analysis in the finite-sample regime, one could possibly utilize the theory of information bottleneck~\cite{saxe2019information, shamir2010learning}, we leave this as future work.
In the next section, we will explore the decomposition strategy selection and propose an automatic graph decomposition algorithm for arbitrary graph-structured data.
\newpage
\vspace{-0.5em}
\section{Graph Decomposition for General Graph-structured Data}
\vspace{-0.5em}

\begin{wrapfigure}{r}{0.65\textwidth}
    \centering
    \vspace{-1em}
    \subfigure[Training Accuracy]{
    \includegraphics[width=0.45\linewidth]{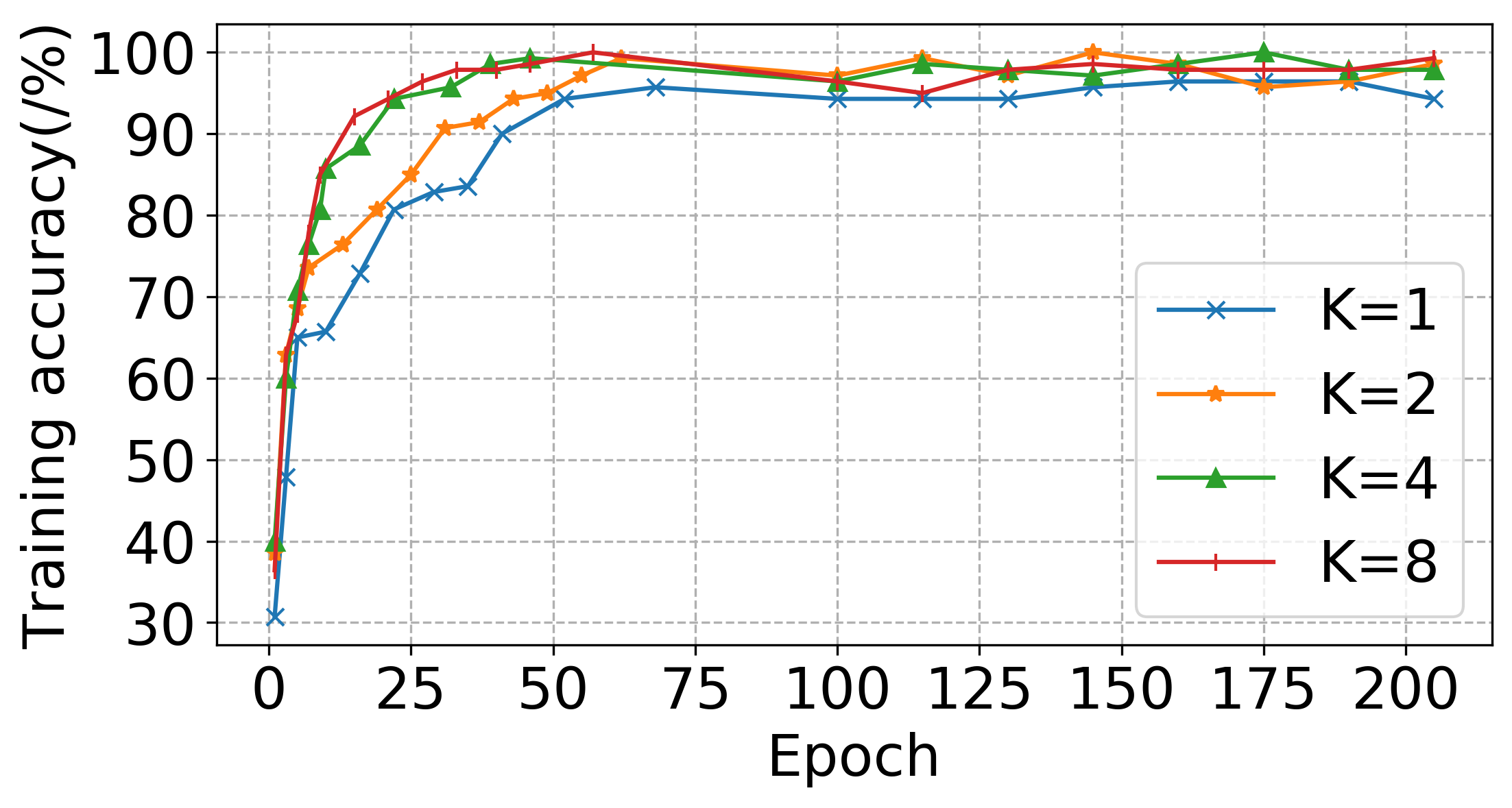}
    \label{fig:train}}
    \subfigure[Testing Accuracy]{
    \includegraphics[width=0.45\linewidth]{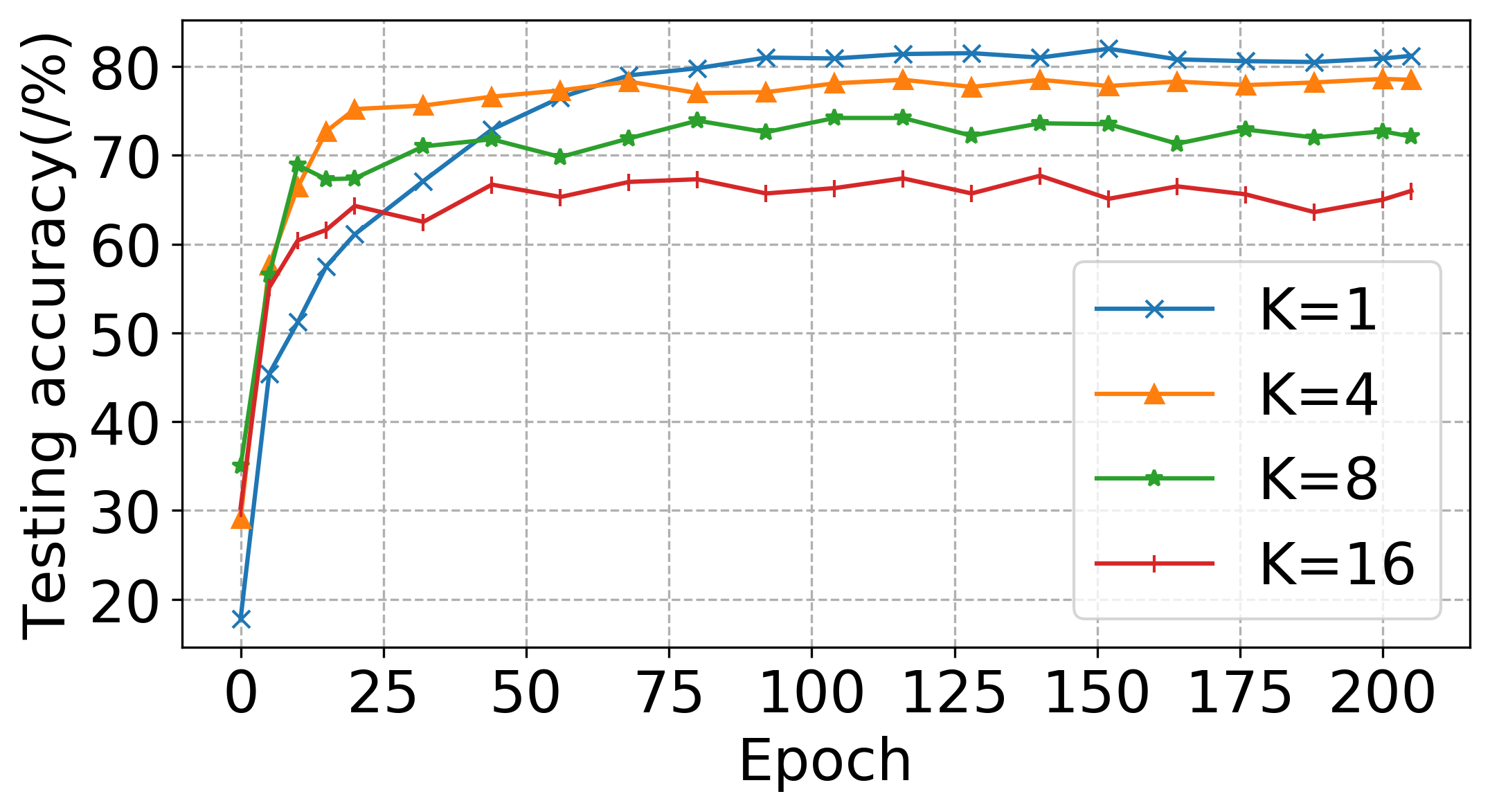}
    \label{fig:test}}
    \caption{Performance of random graph decomposition on Cora.}\label{fig:graph}
    \vspace{-1em}
\end{wrapfigure}
Despite the theoretical merits of graph decomposition, it is non-trivial to perform decomposition on arbitrary graph-structured data. Clearly, there is no absolute geometric space and direction concept in most real-world graphs, and the \textit{spatial anisotropy}~\cite{knyazev2018spectral} makes decomposing an image with the predefined coordinates and directions much easier than graph. As such, we first explore a naive decomposition method (random graph decomposition) and analyze its weakness. Then, we propose a more reasonable graph decomposition strategy to address those weakness.




\vspace{-0.5em}
\subsection{Random graph decomposition}
\vspace{-0.5em}

A simple graph decomposition strategy is to randomly decompose the adjacency matrix $\bm{A}$ into $K$ pieces, i.e., directly distribute the edges from the original graph into the subgraphs, then rebuild every layer with $K$ weight matrices. To show the weakness of this method, we empirically evaluate its performance on the Cora dataset~\cite{DBLP:conf/iclr/KipfW17} (experimental details are in the Appendix~\ref{apdx:setup}). Results are shown in Figure~\ref{fig:graph}. Setting $K$ to $1$ corresponds to the original GCN. It is observed that increasing the number of decomposition components increases the convergence speed but leads to lower testing accuracy. Two possible explanations for this phenomenon are: (1) More weight matrices brought by decomposition may cause overfitting; (2) Random decomposition may break the graph connectivity and impede the spread of information as GCN relies on the graph structures to propagate the node features and labels along the edges. Random decomposition may lead to the result that the nodes can be trapped into a smaller region and cannot spread to distant reachable nodes in the original graph. 
Therefore, an inappropriate decomposition could reduce the graph connectivity and affects the information propagation among nodes, thus, decreasing the model learning ability.
Reducing $K$ for different GCN layers can partly alleviate the overfitting problem, but how to maintain the graph connectivity remains a challenge.  

\vspace{-0.5em}
\subsection{DeGNN: Connectivity-aware graph decomposition}
\vspace{-0.5em}

\begin{wrapfigure}{r}{0.55\textwidth}
\vspace{-2em}
\centering
\begin{minipage}[t]{0.55\textwidth}
\begin{algorithm}[H]
\caption{Connectivity-aware graph decomposition}
\small
\label{alg:all}
\begin{flushleft}
\textbf{Input}: The graph $\mathcal{G}=(\mathcal{V},\mathcal{E})$, $N=|\mathcal{V}|$.\\
\textbf{Parameter}: The number of partitions $p$ for METIS, the number of decomposed graphs $K$.\\
\textbf{Output}: The decomposed graph ($\mathcal{G}_1, \mathcal{G}_2,...,\mathcal{G}_K)$.
\end{flushleft}
\begin{algorithmic}[1] 
\STATE Partition the graph $\mathcal{G}$ into $p$ subgraphs with METIS.
\STATE Merge the subgraphs into $\mathcal{G}_m$.
\STATE $T \gets$ Generate a random spanning forest on $\mathcal{G}_m$.
\STATE $\forall i\in[1,K]$, $R_i\gets(\mathcal{V},\varnothing)$.
\STATE $p\gets0$.
\FOR{$i=1$ \textbf{to} $N$}
\FOR{$v_j\in$ Neighbor($v_i$) on the residual graph $\mathcal{G}/T$}
\STATE Assign edge $(v_i, v_j)$ to $R_{p+1}$
\STATE $p\gets(p+1)\%K$
\ENDFOR
\ENDFOR
\STATE \textbf{return} $(R_1\cup T, R_2\cup T,..., R_K\cup T)$
\end{algorithmic}
\end{algorithm}
\end{minipage}
\vspace{-1em}
\end{wrapfigure}
Inspired by our theoretic analysis on leveraging graph decomposition to preserve node feature information along with different GNN layers and the drawbacks of random graph decomposition, we propose the DeGNN to automatically perform graph decomposition on general graph structured data.
Different with random decomposition, to take the graph connectivity into account, we propose to utilize the spanning tree structure for preserving the accessibility of the nodes. As shown in Algorithm~\ref{alg:all}, we first generate the spanning forest of the graph (line 3), and the replicas of the graph skeleton $T$ will be distributed to the decomposed graphs (line 12). In this way, the node connectivity is still preserved after the decomposition. 
To control the graph connectivity of the generated spanning tree structures, we use METIS~\cite{karypis1998fast} to eliminates some edge cuts before generating the skeletons. The hyperparameter $p$ in METIS controls the amounts of edge cuts and further leads to different levels of node connectivity (lines 1-2). Finally, we decompose the residual graph (lines 4-11), and for each node the adjacent nodes and the associated edges are distributed in the decomposition graphs uniformly and randomly (lines 7-10). 

The advantages of this proposed connectivity-aware graph decomposition are: (1) it will not generate independent subgraphs such that the information propagation process is not blocked; (2) it can ease the overfitting problem of random graph decomposition.

\vspace{-0.5em}
\section{Experiments}
\vspace{-0.5em}



\begin{table}[t]

\centering
\caption{Dataset Statistics} \label{Dataset}
\scriptsize
\begin{tabular}{crrrcccc}
\toprule
\textbf{Dataset}&\textbf{\#Nodes}& \textbf{\#Features}&\textbf{\#Edges}&\textbf{\#Classes}&\textbf{\#Train/Val/Test}&\textbf{Task type}\\
\midrule
Cora& 2,708 & 1,433 &5,429&7& 140/500/1,000 & Transductive\\
Citeseer& 3,327 & 3,703&4,732&6& 120/500/1,000 & Transductive\\
Pubmed& 19,717 & 500 &44,338&3& 60/500/1,000 & Transductive\\
\midrule
Company Small& 96,532  & 114& 1,013,936 & 2 &709,755/101,393/202,788 &Transductive  \\
Company Large& 126,327 & 480& 5,001,222  & 2 &3,500,855/500,122/1,000,245 &Transductive  \\
Amazon Computer& 13,381  & 767& 245,778 & 10 &200/300/12,881&Transductive\\
Amazon Photo &7,487  & 745& 119,043 & 8 & 160/240/7,087&Transductive\\
Coauthor CS& 18,333  & 6,805 & 81,894 & 15 & 300/450/17,583 & Transductive\\
Coauthor Physics& 34,493 & 8,415 & 247,962 & 5 & 100/150/34,243 & Transductive\\
Actor & 7,600 & 931 & 33,544 & 5 & 3,648/608/760 & Transductive \\
\midrule
Flickr& 89,250 & 500 & 899,756 & 7 & 44,625/22,312/22,312 & Inductive\\ 
Reddit& 232,965 & 602 & 11,606,919 & 41 & 155,310/23,297/54,358 & Inductive \\ 
\bottomrule
\end{tabular}
\end{table}

\begin{table}[t!]
\centering
\caption{Test accuracy (in \%) on the benchmark datasets. $^*$ indicates that we ran our own implementation.
We 
use bold font for methods with the highest 
average accuracy.}
\label{Node}
\noindent
\scriptsize
\begin{tabular}{lccc|cc}
\toprule
\multirow{2}{*}{\textbf{Models}} & \multicolumn{3}{c}{\textbf{Transductive}} & \multicolumn{2}{c}{\textbf{Inductive}}\\\cmidrule(lr){2-4} \cmidrule(lr){5-6} & \textbf{Cora}& \textbf{Citeseer}&\multicolumn{1}{c}{\textbf{Pubmed}}&\textbf{Reddit}&\textbf{Flickr} \\
\midrule
GPNN& 81.8 & 69.7&79.3 &  \multicolumn{2}{c}{\multirow{2}{*}{GraphSAGE}}\\
NGCN& 83.0 & 72.2 &79.5  \\
DGCN& 83.5 & 72.6&80 & 95.4$\pm$0.0 & 50.1$\pm$1.3\\
\cline{5-6} DropEdge& 82.8 & 72.3 & 79.6 &  \multicolumn{2}{c}{\multirow{2}{*}{DropEdge}}\\
DGI& 82.3$\pm$0.6 & 71.8$\pm$0.7 &76.8$\pm$0.6 \\
GMI& 82.7$\pm$0.2 & 73.0$\pm$0.3 & \textbf{80.1$\pm$0.2} & \textbf{96.7$\pm$0.0} & 51.9$\pm$0.0$^*$ \\
\cline{5-6} GAT& 83.0$\pm$0.7 & 72.5$\pm$0.7 &79.0$\pm$0.3 &\multicolumn{2}{c}{\multirow{2}{*}{FastGCN}} \\
LGCN& 83.3$\pm$0.5 & 73.0$\pm$0.6 &79.5$\pm$0.2 \\
APPNP& 83.3$\pm$0.5 & 71.8$\pm$0.5 & \textbf{80.1$\pm$0.2} & 93.7$\pm$0.0 & 50.4$\pm$0.1 \\
\midrule
GCN$^*$& 81.8$\pm$0.5 & 70.8$\pm$0.5 &79.3$\pm$0.7 & 95.7$\pm$0.0 & 49.2$\pm$0.3 \\
JK-Net$^*$& 81.8$\pm$0.5  & 70.7$\pm$0.7 & 78.8$\pm$0.7 & 96.4$\pm$0.1 & 51.9$\pm$0.1 \\
ResGCN$^*$& 82.2$\pm$0.6 & 70.8$\pm$0.7 & 78.3$\pm$0.6 & 96.3$\pm$0.1 & 51.5$\pm$0.1 \\
DenseGCN$^*$& 82.1$\pm$0.5 & 70.9$\pm$0.8 &79.1$\pm$0.9 & 96.4$\pm$0.0 & 52.1$\pm$0.0 \\
\midrule
\textbf{DeGNN(GCN)$^*$}& 83.7$\pm$0.4 & 72.5$\pm$0.3 & 79.8$\pm$0.6 &  96.4$\pm$0.0 & 51.5$\pm$0.2 \\
\textbf{DeGNN(JK)$^*$}& 84.1$\pm$0.3 & \textbf{73.1$\pm$0.5} & 80.0$\pm$0.4 & 96.6$\pm$0.0 & \textbf{52.5$\pm$0.0} \\
\textbf{DeGNN(Res)$^*$}& 83.9$\pm$0.5 & 72.6$\pm$0.4 & 79.9$\pm$0.5 & \textbf{96.7$\pm$0.1} & 51.9$\pm$0.1 \\
\textbf{DeGNN(Dense)$^*$}& \textbf{84.3$\pm$0.3} & 72.7$\pm$0.5 & \textbf{80.1$\pm$0.7} & 96.6$\pm$0.0 & \textbf{52.5$\pm$0.0} \\
\bottomrule
\end{tabular}
\vspace{-1em}
\end{table}

\label{sec:exp}
We conduct experiments on widely used benchmark datasets to validate the effectiveness of our method in both transductive and inductive settings. An overview summary of statistics of the datasets is given in Table~\ref{Dataset}. We leave the detailed description
of the datasets, detailed implementation settings and the hyperparameter search procedure in the Appendix~\ref{apdx:dataset} and \ref{apdx:setup}.

\paragraph{Comparison with state-of-the-art}
We compare our method with the representative methods in recent years, including shallow models such as GCN~\cite{DBLP:conf/iclr/KipfW17}, GPNN~\cite{DBLP:conf/iclr/LiaoBTGUZ18}, NGCN~\cite{DBLP:conf/uai/Abu-El-HaijaKPL19}, DGCN~\cite{zhuang2018dual},  DGI~\cite{DBLP:conf/iclr/VelickovicFHLBH19}, GMI~\cite{DBLP:conf/www/PengHLZRXH20}, GAT~\cite{velivckovic2017graph}, LGCN~\cite{gao2018large}, and APPNP~\cite{DBLP:conf/iclr/KlicperaBG19}; and deeper models such as JK-Net~\cite{DBLP:conf/icml/XuLTSKJ18}, ResGCN~\cite{DBLP:conf/iclr/KipfW17}, DenseGCN~\cite{DeeperInsight}, DropEdge~\cite{rong2019dropedge}. 
We also compare our method with the inductive methods such as GraphSAGE~\cite{hamilton2017inductive}, DropEdge and FastGCN~\cite{DBLP:conf/iclr/ChenMX18} on two larger graph dataset including Flickr and Reddit.

Since we can apply the decomposition 
techniques to a range of base models,
we use DeGNN(GCN), DeGNN(JK),
DeGNN(Res), and DeGNN(Dense) to denote
the method that applies our
decomposition algorithm to vanilla GCN,
JK-Net, ResGCN, and DenseGCN.


Table~\ref{Node} and \ref{tab:MoreNode} summarizes the test accuracy of the baselines and our approaches. On Cora, Citeseer, and Pubmed, DeGNN(GCN) achieves significantly 
better performance. On many datasets,
simply adding the decomposition step
on GCN can lead to better performance
even better than more recent state-of-the-art models!
Moreover, 
with the help of deeper architectures, DeGNN with more advanced base model can outperform current state-of-the-art methods. Specifically, DeGNN(Dense) achieves a remarkable  84.3\% testing accuracy with 5 layers on Cora.

We further evaluate DeGNN on a variety of other datasets such as Coauthor CS, Coauthor Physics, Amazon Computers and Amazon Photo, and two new real-world dataset Company Small and Company Large.
The results demonstrate that, in general, DeGNN outperform GCN, JK-Net, ResGCN and DenseGCN and GAT, and the base models can benefit from DeGNN.

Besides the transductive tasks, we also evaluate the DeGNN on the inductive ones. As it is not 
well suitable for the standard GCN setting, we add an additional decomposition step when involving the validation set and the testing set. It is excited to see that DeGNN can still achieve competitive results. We think it is an interesting future work to design an end-to-end framework that can automatically combine DeGNN with graph sampling based methods for inductive scenarios.

\begin{figure*}[t]
\centering
\subfigure{
\scalebox{.31}{
\includegraphics[width=0.9\linewidth]{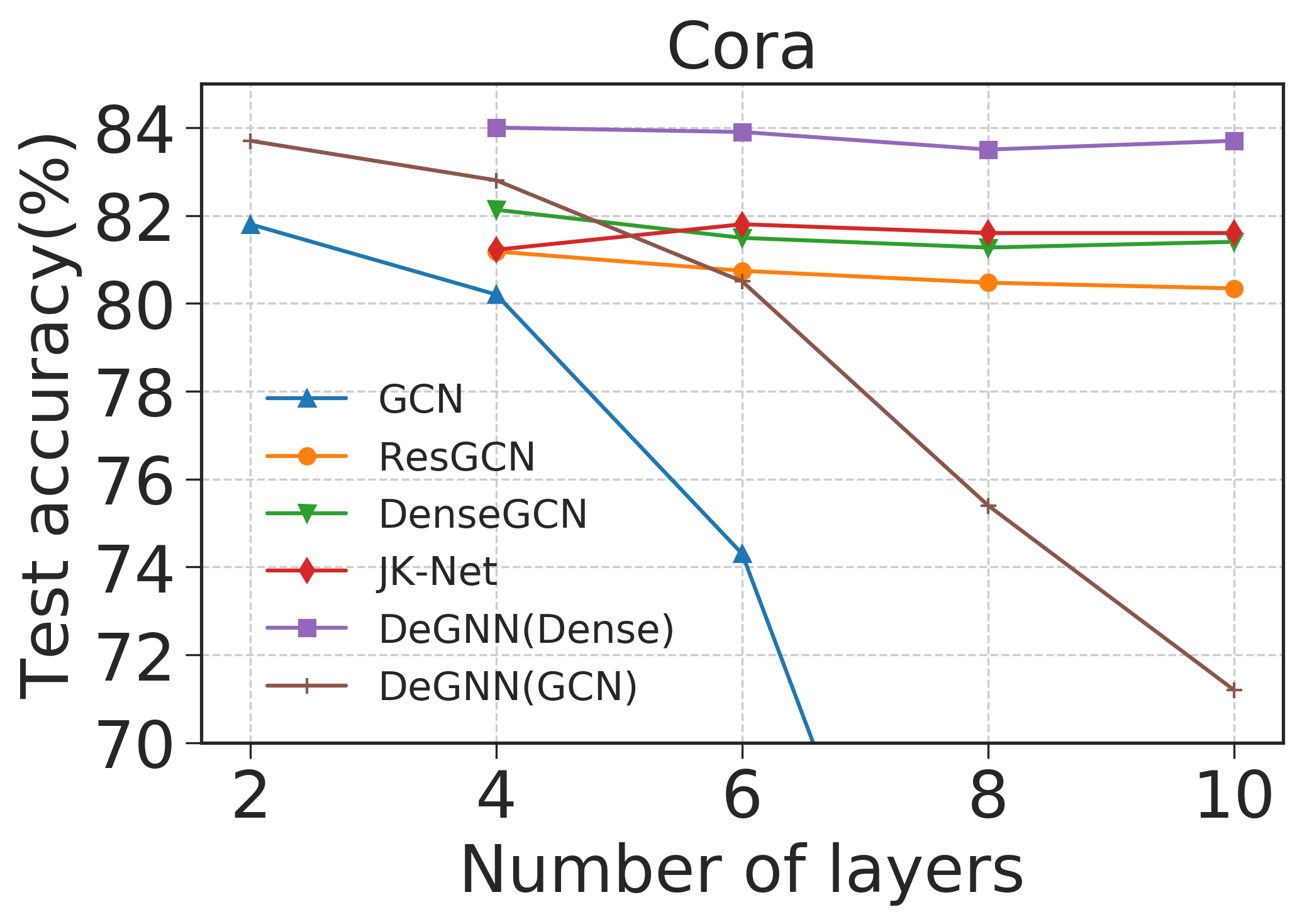}
}
}
\subfigure{
\scalebox{.31}{
\includegraphics[width=0.9\linewidth]{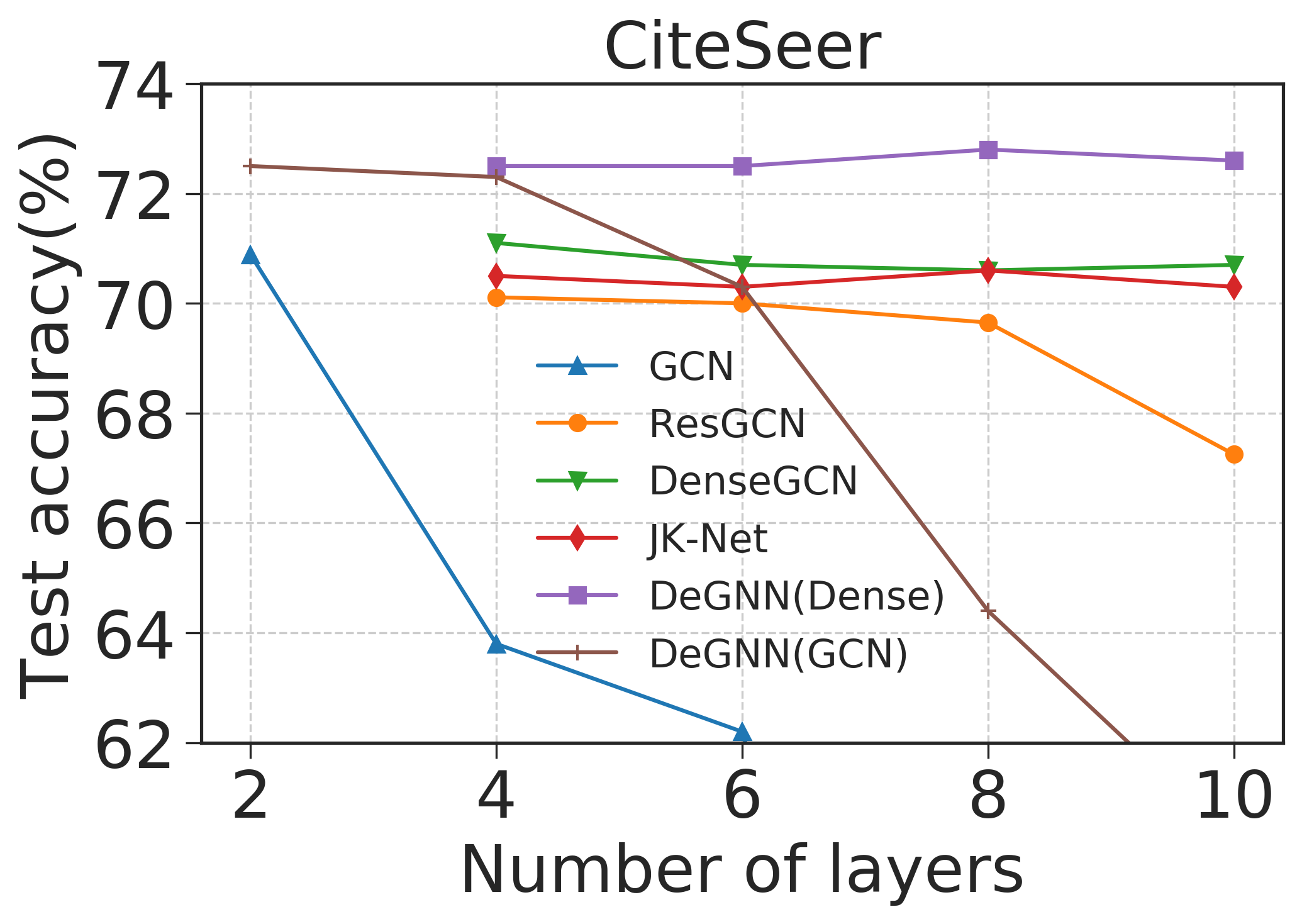}
}
}
\subfigure{
\scalebox{.31}{
\includegraphics[width=0.9\linewidth]{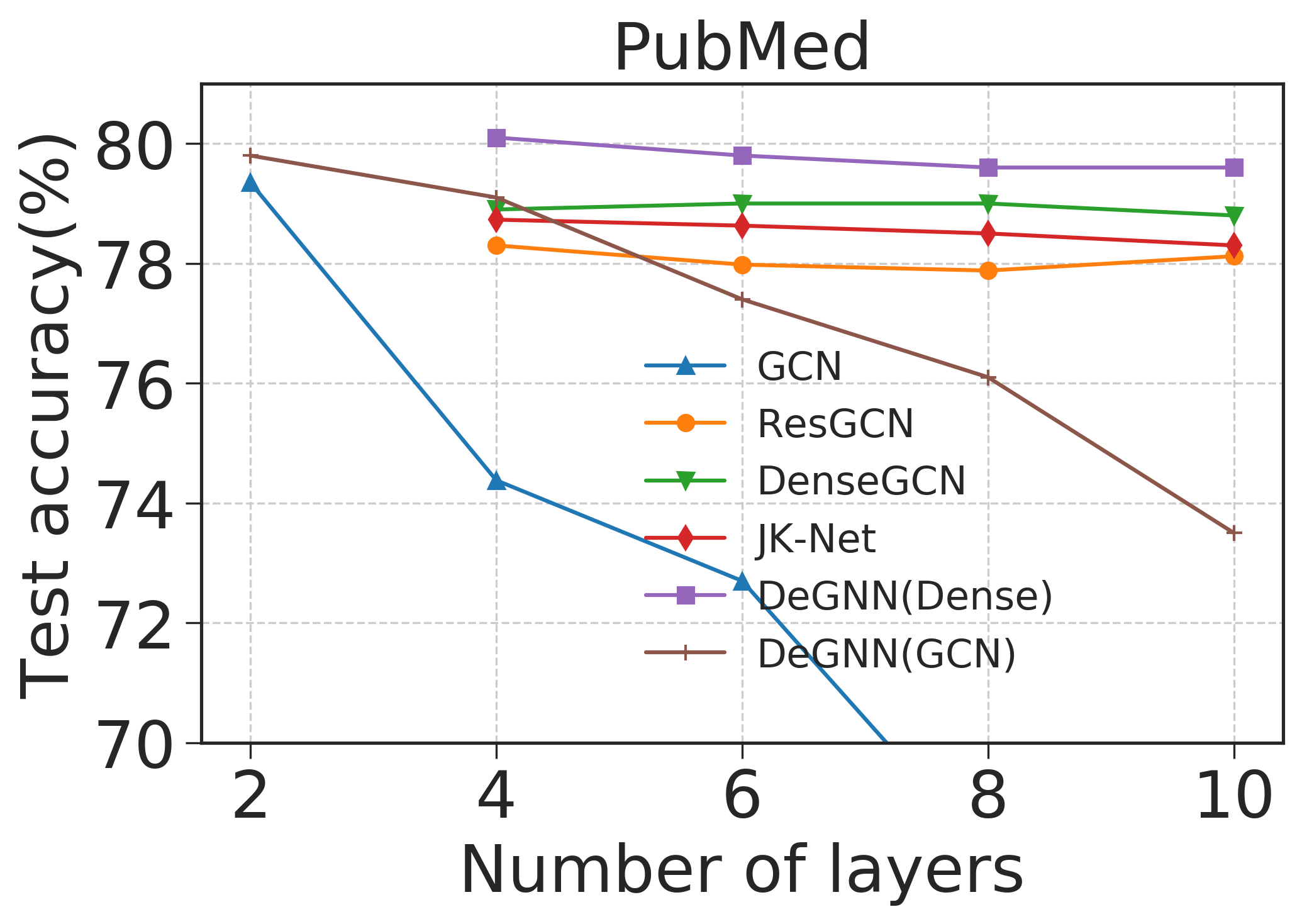}
}
}
\caption{Influence of model depth (number of layers) on classification performance. More details with other backbones are reported in the Appendix~\ref{apdx:depth}.}
\label{Deep}
\end{figure*}


\paragraph{Analysis on the deep architecture.} Here, we investigate the influence of model depth (number of layers) on classification performance on the three citation datasets. We compare DeGNN(GCN) and DeGNN(Dense) with ResGCN, JK-Net, and DenseGCN. 
When the model depth is two, all baselines degenerate to the original 2-layer GCN model.
As shown in Figure~\ref{Deep}, for the original GCN, it gets the best results with a 2-layer model and its performance decreases rapidly with the increase of layers.
For ResGCN, DenseGCN, and JK-Net, they can keep more information on the original features compared with GCN and get a relatively good performance, but perform much worse than DeGNN(Dense). Even with 10 layers, the performance of DeGNN does not decrease as the other baselines do and outperform their best results on all datasets.

\begin{wrapfigure}{r}{0.4\textwidth}
\vspace{-2em}
    \centering
    \includegraphics[width=.9\linewidth]{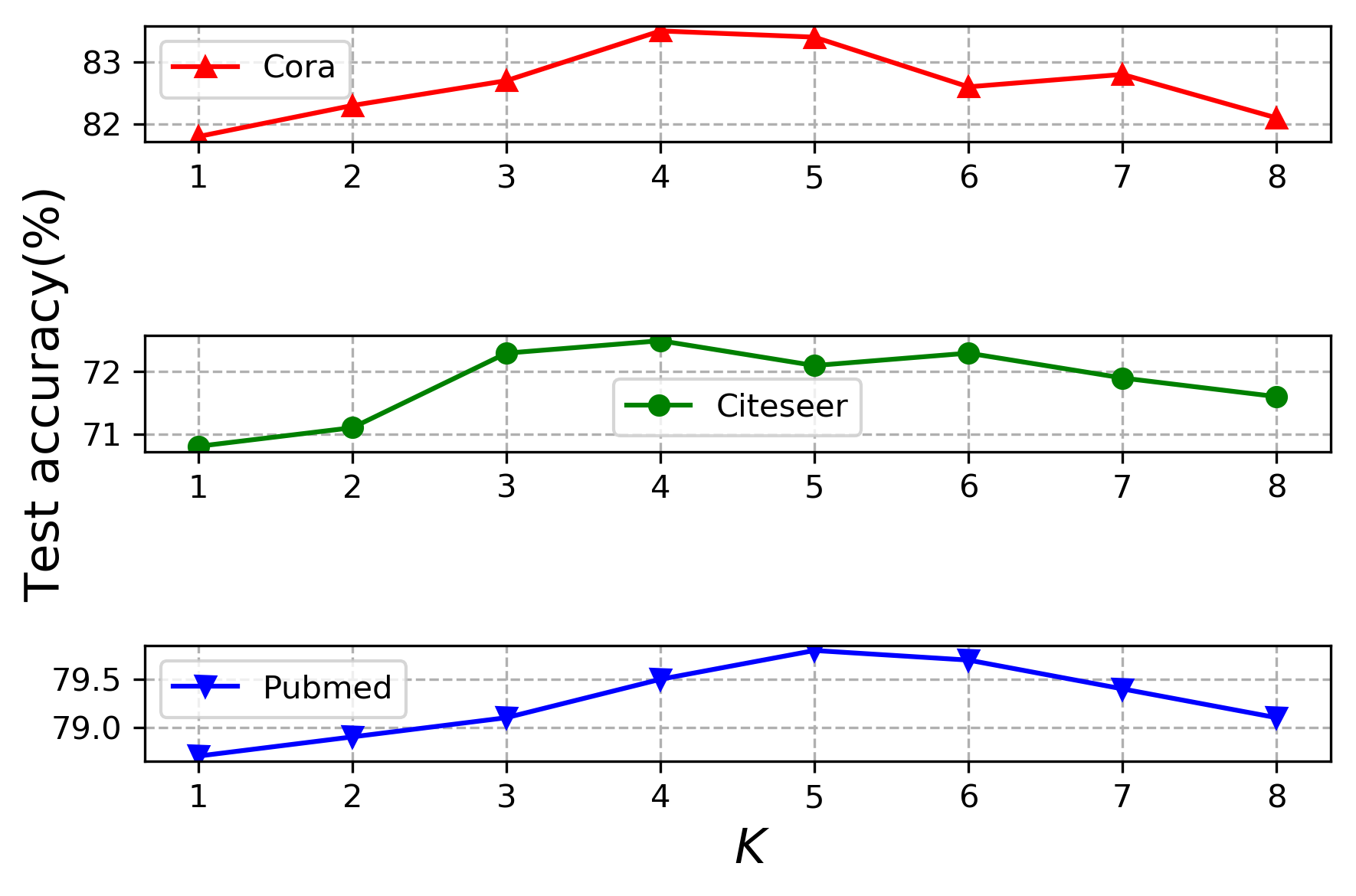}
    \caption{Test accuracy for different number of decomposed pieces $K$.}
\label{fig:selectn}
    \vspace{-1em}
\end{wrapfigure}

\paragraph{Impact of the decomposition parameter $K$.}
The number of decomposed subgraphs $K$ is an important parameter in our framework.
To analyze its influence, we conduct an experiment on three citation networks 
and Figure~\ref{fig:selectn} illustrates 
the result.
Here we set the skeleton $T$ to $\varnothing$ so that the decomposition strategy dominates the model performance. As we can see, the best number of decomposed subgraphs $K$ for Cora and Citeseer is 4 and it is 5 for Pubmed. As $K$ grows from 1, the test accuracy increases until it reaches the maximum point and it decreases when $K$ is larger. These results imply that with our spanning-tree-based sampling framework, there is an optimal graph decomposition parameter $K$ for better model performance. More analysis on the graph connectivity are in the Appendix~\ref{apdx:connect}.

\begin{table}[t!]
\centering
\caption{Test accuracy (in \%) on other datasets. $^*$ indicates that we ran our own implementation. We use AUC on the Company Small/Large because of class imbalance. We 
use bold font for methods with the highest 
average accuracy.}
\label{tab:MoreNode}
\scriptsize
\noindent
\begin{tabular}{lccccccc}
\toprule
\textbf{Models} & {\textbf{\makecell{Company \\Small}}}& {\textbf{\makecell{Company \\Large}}} & {\textbf{\makecell{Amazon \\Computer}}} & 
{\textbf{\makecell{Amazon \\Photo}}} & {\textbf{\makecell{Coauthor \\CS}}} & 
{\textbf{\makecell{Coauthor \\Physics}}} &
\textbf{Actor}\\
\midrule
GAT$^*$& 70.4$\pm$0.5 & 80.6$\pm$0.5 & 80.1$\pm$0.6 & 85.7$\pm$1.0 & 87.4$\pm$0.2 & 90.2$\pm$1.4 & 27.7$\pm$0.5\\
GCN$^*$& 73.1$\pm$0.6& 80.5$\pm$0.4 & 82.4$\pm$0.4 & 85.9$\pm$0.6 & \textbf{90.7$\pm$0.2} & 92.7$\pm$1.1 & 27.0$\pm$0.8\\
JK-Net$^*$& 71.9$\pm$0.3& 80.7$\pm$0.4 & 82.0$\pm$0.6 & 85.9$\pm$0.7 & 89.5$\pm$0.6 & 92.5$\pm$0.4 & 24.7$\pm$0.9\\
ResGCN$^*$& 73.0$\pm$0.5& 80.2$\pm$0.4 & 81.1$\pm$0.7 & 85.3$\pm$0.9 & 87.9$\pm$0.6 & 92.2$\pm$1.5 & 27.0$\pm$1.2 \\
DenseGCN$^*$& 73.5$\pm$0.3& 80.9$\pm$0.2 & 81.3$\pm$0.9 & 84.9$\pm$1.1 & 88.4$\pm$0.8 & 91.9$\pm$1.4 & 26.2$\pm$0.8\\
\midrule
\textbf{DeGNN(GCN)$^*$}& 71.6$\pm$0.3&  81.1$\pm$0.2 & 82.8$\pm$0.6 & \textbf{86.3$\pm$0.4} & 89.5$\pm$0.5 & 92.4$\pm$0.5 & \textbf{29.2$\pm$0.5}\\
\textbf{DeGNN(JK)$^*$}& 72.4$\pm$0.4& 80.7$\pm$0.3 & 82.5$\pm$0.7 & 86.1$\pm$0.7 & 90.5$\pm$0.4 & 92.2$\pm$0.5 & 28.5$\pm$0.5\\
\textbf{DeGNN(Res)$^*$}&\textbf{ 74.0$\pm$0.4}& 81.1$\pm$0.4 & 82.5$\pm$0.5 & 85.8$\pm$0.9 & 90.1$\pm$0.2 & \textbf{92.9$\pm$0.6} & 28.6$\pm$0.8\\
\textbf{DeGNN(Dense)$^*$}& 73.7$\pm$0.6& \textbf{81.2$\pm$0.3} & \textbf{83.1$\pm$0.5} & 86.2$\pm$0.8 & 90.2$\pm$0.1 & 92.1$\pm$1.7 & 28.8$\pm$0.9\\
\bottomrule
\end{tabular}
\end{table}

\vspace{-0.5em}
\section{Conclusion}
\vspace{-0.5em}

In this paper, we investigated the importance of graph decomposition in graph neural networks. We theoretically verified that graph decomposition can help avoid the information loss problem caused by increasing networks depth. To utilize the information preserving ability of the decomposition in general graph-structured data , we introduce a novel connectivity-aware graph decomposition to balance the trade-off between information loss and model performance of GNNs. We conducted extensive experiments on ten datasets and analyzed the property of our model. Our model achieves state-of-the-art performances and could better preserve information with deeper architectures.

\bibliographystyle{abbrvnat}
\bibliography{arxiv-gcn}

\newpage
\appendix
\section{Proofs}
We begin by introducing our notation. Hereafter, scalars will be written in italics, vectors in bold lower-case and matrices in bold upper-case letters. For an $m\times n$ real matrix ${\bf A}$, the matrix element in the $i$th row and $j$th column is denoted as $({\bf A})_{ij}$, and $i$th entry of a vector ${\bf a} \in \mathbb{R}^{m}$ by $({\bf a})_i$. Also, $j$th column of ${\bf A}$ is denoted by $({\bf A})_j$, or $({\bf A})_{[i=1, 2, \dots, m], j}$. Similarly, we denote $i$th row by $({\bf A})_{i, [j=1, 2, \dots, n]}$. The inner product between two vectors $({\bf A})_i$ and $({\bf A})_{i'}$ is denoted by $\langle ({\bf A})_i ,({\bf A})_{i'}\rangle$. 

We vectorize a matrix ${\bf A}$ by concatenating its columns such that
\[
  \vectorized({\bf A}) =
  \left[ {\begin{array}{c}
  ({\bf A})_1 \\
  ({\bf A})_2 \\
  \vdots \\
  ({\bf A})_n \\
  \end{array} } \right]
\]
and denote it by $\vectorized({\bf A})$. For matrices ${\bf A}\in \mathbb{R}^{m\times n}$ and ${\bf B}\in \mathbb{R}^{k\times l}$, we denote the kronecker product of ${\bf A}$ and ${\bf B}$ by ${\bf A}\otimes {\bf B}$ such that
\[
  {\bf A} \otimes {\bf B} =
  \left[ {\begin{array}{ccc}
  ({\bf A})_{11}{\bf B} & \hdots & ({\bf A})_{1n}{\bf B} \\
  \vdots & \ddots & \vdots \\
  ({\bf A})_{m1}{\bf B} & \hdots & ({\bf A})_{mn}{\bf B} \\
  \end{array} } \right].
\]

Note that ${\bf A} \otimes {\bf B}$ is of size ${mk\times nl}$.

Next, we list some existing results which we require repeatedly throughout this section.

\paragraph*{Preliminaries.}
\begin{enumerate}
    \item Suppose ${\bf A}\in \mathbb{R}^{m\times n}$, ${\bf B}\in \mathbb{R}^{n\times k}$ and ${\bf C}\in \mathbb{R}^{k\times p}$. We have
    \begin{equation}\label{Equation: Vectorized Triple Product}
       \vectorized({\bf ABC}) = ({\bf C}^T\otimes {\bf A})\vectorized({\bf B}).
    \end{equation}
    \item Let ${\bf A}\in \mathbb{R}^{m\times n}$, ${\bf B}\in \mathbb{R}^{n\times k}$ and ${\bf C}\in \mathbb{R}^{m'\times n'}$, ${\bf D}\in \mathbb{R}^{n'\times k'}$
    \begin{equation}\label{Equation: Product Kronecker}
       ({\bf AB}\otimes {\bf CD}) = ({\bf A}\otimes {\bf C})({\bf B}\otimes {\bf D}).
    \end{equation}
    \item For ${\bf A}\in \mathbb{R}^{m\times m}$ and ${\bf B}\in \mathbb{R}^{n\times n}$, singular values of ${\bf A} \otimes {\bf B}$ is given by $\lambda_i({\bf A})\lambda_j({\bf B})$, $i=1, 2, \dots, m $ and $j=1, 2, \dots, n$.
    \item Let ${\bf x}$ and ${\bf y}$ be an $n$-dimensional random vector defined over finite alphabets $\mathcal{X}^n$ and $\Omega^n$, respectively. 
    We denote entropy of {\bf x} by $\mathcal{H}({\bf x})$ and mutual information between {\bf x} and {\bf y} by $\mathcal{I}({\bf x}; {\bf y})$. We list the followings:
    \begin{equation}\label{Equation: Information measures for fx}
    \begin{split}
        \mathcal{H}(f({\bf x})) &\stackrel{(a)}{\leq} \mathcal{H}({\bf x})\\
        \mathcal{I}({\bf x}; f({\bf y})) &\stackrel{(b)}{\leq} \mathcal{I}({\bf x}; {\bf y}) 
    \end{split}
    \end{equation}
    such that $f:\mathbb{R}\rightarrow\mathbb{R}$ is some deterministic function, and equality holds for both inequalities \textit{iff} $f$ is bijective.
    \item As introduced in Section~\ref{Section: Theoretical Analysis}, for GCN, we have:
    \begin{equation}\label{Equation: GCN Single layer}
    \begin{split}
        {\bf Y}^{(i+1)} = f_{{\bf A},{\bf W}^{(i+1)}}({\bf Y}^{(i)})= \sigma ({\bf A}{\bf Y}^{(i)}{\bf W}^{(i+1)}).
    \end{split}
    \end{equation}
    \item For GraphCNN, let now ${\bf A} \in \mathbb{R}^{n\times n}$ be decomposed into $K$ additive $n\times n$ matrices such that ${\bf A} = \sum_{k=1}^K{\bf A}_k$. The layer-wise propagation rule becomes: 
    {\small
    \begin{equation}\label{Equation: GCNN Single layer}
    \begin{split}
        {\bf Y}^{(i+1)} = g_{{\bf A}_k,{\bf W}_k^{(i+1)}}({\bf Y}^{(i)})= \sigma \bigg(\sum_{k=1}^K {\bf A}_k{\bf Y}^{(i)}{\bf W}_k^{(i+1)}\bigg).
    \end{split}
    \end{equation}
    }
\end{enumerate}

\paragraph*{Proofs.}
The proofs are listed below in order.
\begin{proof}[Proof of Lemma~\ref{Lemma: Linearize GCN}]
Applying vectorization to the layer-wise propagation rule introduced in~(\ref{Equation: GCN Single layer}), we have
\begin{equation}\label{Equation: GCN layer response}
    \begin{split}
        {\bf y}^{(i+1)} &{=}  \vectorized\big(\sigma({\bf AY}^{(i)}{\bf W}^{(i+1)})\big)\\
        {\bf y}^{(i+1)} &\stackrel{(a)}{=}   \sigma\big(\vectorized({\bf AY}^{(i)}{\bf W}^{(i+1)})\big)\\
        {\bf y}^{(i+1)} &\stackrel{(b)}{=} \sigma\big( (({\bf W}^{(i+1)})^T\otimes {\bf A}) {\bf y}^{(i)} \big)\\
        {\bf y}^{(i+1)} &\stackrel{(c)}{=}{\bf P}^{(i+1)}(({\bf W}^{(i+1)})^T\otimes {\bf A}) {\bf y}^{(i)} 
    \end{split}
\end{equation}
where (a) follows from the element-wise application of $\sigma$, (b) follows from~(\ref{Equation: Vectorized Triple Product}), and (c) results from introducing a diagonal matrix ${\bf P}^{(i+1)}$ with diagonal entries in $\{a, 1\}$ such that $({\bf P}^{(i+1)})_{j, j}=1$ if $\big(({\bf W}^{(i+1)}\otimes {\bf A}){\bf y}^{(i)}\big)_j\geq 0$, and $({\bf P}^{(i+1)})_{j, j}=a$ elsewhere.

By a recursive application of~(\ref{Equation: GCN layer response}c), we have
{\small
\begin{equation*}\label{Equation: GCN Final}
    {\bf y}^{(l)} = {\bf P}^{(l)}({\bf W}^{(l)}\otimes {\bf A}) \dots {\bf P}^{(2)}({\bf W}^{(2)}\otimes {\bf A}) {\bf P}^{(1)}({\bf W}^{(1)}\otimes {\bf A}) {\bf x}.
\end{equation*}
}
\end{proof}
We drop the transpose from ${\bf W}^{(i+1)}$ in order to avoid cumbersome notation. The singular values of ${\bf W}^{(i+1)}$ are our primary interest thereof our results still hold.

Following Lemma~\ref{Lemma: Linearize GCN}, the next key step in our proving is as follows.

\begin{lemma}\label{Lemma: MI General}
Consider the singular value decomposition {\small ${\bf U}{\bf \Lambda}{\bf V}^T={\bf P}^{(l)}({\bf W}^{(l)}\otimes {\bf A}) ... {\bf P}^{(2)}({\bf W}^{(2)}\otimes {\bf A}) {\bf P}^{(1)}({\bf W}^{(1)}\otimes {\bf A})$} such that {\small $({\bf \Lambda})_{j,j}=\lambda_j({\bf P}^{(l)}({\bf W}^{(l)}\otimes {\bf A}) ... {\bf P}^{(2)}({\bf W}^{(2)}\otimes {\bf A}) {\bf P}^{(1)}({\bf W}^{(1)}\otimes {\bf A}))$}, and let $\tilde{\bf x}={\bf V}^T{\bf x}$. We have 
\begin{equation}
    \begin{split}
        \mathcal{I}({\bf x}; {\bf y}^{(l)}) \stackrel{(1)}{=} \mathcal{I}(\tilde{\bf x}; \Lambda\tilde{\bf x}) &\stackrel{(2)}{\leq}\mathcal{H}(\tilde{\bf x}) \stackrel{(3)}{=}\mathcal{H}({\bf x})
    \end{split}
\end{equation}
\end{lemma}
where (1, 3) results from that ${\bf U}$ and ${\bf V}$ are invertible, and equality holds in (2) \textit{iff} ${\bf \Lambda}$ is invertible, i.e., singular values of {\small ${\bf P}^{(l)}({\bf W}^{(l)}\otimes {\bf A}) ... {\bf P}^{(2)}({\bf W}^{(2)}\otimes {\bf A}) {\bf P}^{(1)}({\bf W}^{(1)}\otimes {\bf A})$} are nonzero.

Theorem~\ref{Theorem: GCN information loss}, \ref{Theorem: GCN no information loss}, \ref{Theorem: GCNN Information Loss General} and \ref{Theorem: GCNN Information No Loss General} can easily be inferred from Lemma~\ref{Lemma: MI General}. That is, $\mathcal{I}({\bf x}; {\bf y}^{(l)})=0$ \textit{iff} $\max_j ({\bf \Lambda}^l)_{j, j} =0$ in the asymptotic regime. Similarly, \textit{iff} $\min_j ({\bf \Lambda}^l)_{j, j} >0$, $\mathcal{I}({\bf x}; {\bf y}^{(l)})$ is maximized and given by $\mathcal{H}({\bf x})$, hence  $\mathcal{L}({\bf y}^{(l)})=0$.

In particular Theorem~\ref{Theorem: GCN information loss},~\ref{Theorem: GCNN Information Loss General} and Corollary~\ref{Corollary: GCNN information loss decompA}, i.e., exponential decay to zero, also hold for traditional ReLU with $f: x\rightarrow x^+=\max(0, x)$.

\begin{proof}[Proof of Lemma~\ref{Lemma: MI General}]
Let ${\bf \Sigma}$ be a $n\times n$ matrix with singular value decomposition ${\bf \Sigma} = {\bf U \Lambda V}^T$. Inspired by the derivation for the capacity of deterministic channels introduced by~\cite{Telatar1999CapacityMIMO}, we derive the following
\begin{equation}\label{Equation: Proof of Lemma MI}
    \begin{split}
        \mathcal{I}({\bf x}; {\bf \Sigma x}) &=\mathcal{I}({\bf x}; {\bf U\Lambda V}^T{\bf x})\stackrel{(a)}{=}\mathcal{I}({\bf x}; {\bf \Lambda V}^T{\bf x})\\
        \mathcal{I}({\bf x}; {\bf \Sigma x}) &\stackrel{(b)}{=}\mathcal{I}({\bf V}^T{\bf x}; {\bf \Lambda V}^T{\bf x})\stackrel{(c)}{=}\mathcal{I}(\tilde{\bf x}; {\bf \Lambda}\tilde{\bf x}).
    \end{split}
\end{equation}
(a) and (b) are a result of~(\ref{Equation: Information measures for fx}b) and that ${\bf U}$ and ${\bf V}$ are unitary hence invertible (bijective) transformations. (c) follows from the change of variables $\tilde{{\bf x}} = {\bf V}^T {\bf x}$.

Note that $\mathcal{I}(\tilde{\bf x}; {\bf \Lambda}\tilde{\bf x}) \leq \mathcal{H}({\bf \Lambda}\tilde{\bf x})$. Using~(\ref{Equation: Information measures for fx}a), we further have $\mathcal{H}({\bf \Lambda}\tilde{\bf y})\leq \mathcal{H}(\tilde{\bf x})=\mathcal{H}({\bf x})$ which completes the proof.
\end{proof}

We recall that we are interested in regimes where $\mathcal{I}({\bf x}; {\bf y}^{(l)})=0$ and $\mathcal{L}({\bf y}^{(l)})=0$. In Lemma~\ref{Lemma: MI General}, we show that $\mathcal{I}({\bf x}; {\bf y}^{(l)})=0$ if $\max_j\lambda_j({\bf P}^{(l)}({\bf W}^{(l)}\otimes {\bf A}) \cdots {\bf P}^{(2)}({\bf W}^{(2)}\otimes {\bf A}) {\bf P}^{(1)}({\bf W}^{(1)}\otimes {\bf A}))=0$, and maximized (and given by $\mathcal{H}({\bf x})$) when ${\bf P}^{(l)}({\bf W}^{(l)}\otimes {\bf A}) \cdots {\bf P}^{(2)}({\bf W}^{(2)}\otimes {\bf A}) {\bf P}^{(1)}({\bf W}^{(1)}\otimes {\bf A})$ is invertible. Therefore, maximum and minimum singular values of ${\bf P}^{(l)}({\bf W}^{(l)}\otimes {\bf A}) \cdots {\bf P}^{(2)}({\bf W}^{(2)}\otimes {\bf A}) {\bf P}^{(1)}({\bf W}^{(1)}\otimes {\bf A})$ are of our interest.


\begin{proof}[Proof of Theorem~\ref{Theorem: GCN information loss}]
Let $\sigma_{\bf A} = \max_j\lambda_j({\bf A})$ and $\sigma_{\bf W} = \sup_i \max_j\lambda_j({\bf W}^{(i)})$. That is, given singular values of ${\bf P}^{(i)}$ is in $\{a, 1\}$, $\sup_i \max_j \lambda_j({\bf P}^{(i)}({\bf W}^{(i)}\otimes {\bf A})) = \sigma_{\bf A}\sigma_{\bf W}$. We, moreover, have $\max_j \lambda_j ({\bf P}^{(l)}({\bf W}^{(l)}\otimes {\bf A}) \cdots {\bf P}^{(2)}({\bf W}^{(2)}\otimes {\bf A}) {\bf P}^{(1)}({\bf W}^{(1)}\otimes {\bf A}))\leq (\sigma_{\bf A} \sigma_{\bf W})^l$. Therefore, if $\sigma_{\bf A} \sigma_{\bf W}<1$, by Lemma~\ref{Lemma: MI General} we have $\mathcal{I}({\bf x}; {\bf y}^{(l)})=\mathcal{O}((\sigma_{\bf A} \sigma_{\bf W})^l)$,
and $\lim_{l\rightarrow \infty} \mathcal{I}({\bf x}; {\bf y}^{(l)})=0$. 
\end{proof}

\begin{sloppypar}
\begin{proof}[Proof of Theorem~\ref{Theorem: GCN no information loss}]
We now denote $\gamma_{\bf A}=\min_j \lambda_j({\bf A})$ and $\gamma_{\bf W} = \inf_{i} \min_{j} \lambda_j({\bf W}^{(i)})$. Hence $\inf_i \min_j \lambda_j({\bf P}^{(i)}({\bf W}^{(i)}\otimes {\bf A}))= a\gamma_{\bf A} \gamma_{\bf W}$. Moreover, $\min_j \lambda_j ({\bf P}^{(l)}({\bf W}^{(l)}\otimes {\bf A}) \cdots {\bf P}^{(2)}({\bf W}^{(2)}\otimes {\bf A}) {\bf P}^{(1)}({\bf W}^{(1)}\otimes {\bf A}))\geq (a \gamma_{\bf A} \gamma_{\bf W})^l$. If $a\gamma_{\bf A} \gamma_{\bf W} \geq1$, $\min_j  \lambda_j ({\bf P}_l({\bf W}^{(l)}\otimes {\bf A}) \cdots {\bf P}_{2}({\bf W}^{(2)}\otimes {\bf A}) {\bf P}_{1}({\bf W}^{(1)}\otimes {\bf A}))\geq1$ $\forall l\in \mathbb{N}^+$, hence $\mathcal{I}({\bf x}; {\bf y}^{(l)})=\mathcal{H}({\bf x})$ and $\mathcal{L}({\bf y}^{(l)})=0$ results by Lemma~\ref{Lemma: MI General}.
\end{proof}
\end{sloppypar}

\begin{proof}[Proof of Corollary~\ref{Corollary: GCN Laplacian}]
Let ${\bf D}$ denote the degree matrix such that $({\bf D})_{j, j}= \sum_m ({\bf A})_{j, m}$, and ${\bf L}$ be the associated normalized Laplacian ${\bf L}={\bf D}^{-1/2}{\bf A}{\bf D}^{-1/2}$. Due to the property of normalized Laplacian such that $\max_j \lambda_j({\bf L})=1$, we have $\sigma_{\bf A}=1$. Inserting this into Theorem~\ref{Theorem: GCN information loss}, the corollary results.
\end{proof}

Similarly as in (\ref{Equation: GCN layer response}), ${\bf y}^{(i+1)}$ can be derived from (\ref{Equation: GCNN Single layer}) as follows:
\begin{equation}\label{Equation: GCNN layer response}
    \begin{split}
        {\bf y}^{(i+1)} &{=}  \vectorized\big(\sigma(\sum_k {\bf A}_k{\bf Y}^{(i)}{\bf W}_k^{(i+1)})\big)\stackrel{(a)}{=}   \sigma(\sum_k \vectorized({\bf A}_k{\bf Y}^{(i)}{\bf W}_k^{(i+1)})\big)\\
        {\bf y}^{(i+1)} &\stackrel{(b)}{=}   \sigma(\sum_k ({\bf W}_k^{(i+1)}\otimes {\bf A}_k){\bf y}^{(i)}\sigma)\stackrel{(c)}{=}  {\bf P}^{(i+1)}\sum_k ({\bf W}_k^{(i+1)}\otimes {\bf A}_k){\bf y}^{(i)}
    \end{split}
\end{equation}
where ${\bf P}^{(i+1)}$ is a diagonal matrix with diagonal entries in $\{a, 1\}$ with $a\in (0, 1)$ such that $({\bf P}^{(i)})_{j, j}=1$ if $\big(\sum_k({\bf W}_k^{(i+1)}\otimes {\bf A}){\bf y}^{(i)}\big)_j\geq 0$, and $({\bf P}^{(i)})_{j, j}=a$ otherwise.

Therefore, ${\bf y}^{(l)}$ is given by 
{\small
\begin{equation*}\label{Equation: GCNN all layers}
\begin{split}
    {\bf y}^{(l)} = {\bf P}^{(l)}\sum_{k_l} ({\bf W}_{k_l}^{(l)}\otimes {\bf A}_{k_l})\cdots {\bf P}^{(2)}\sum_{k_2} ({\bf W}_{k_2}^{(2)}\otimes {\bf A}_{k_2}){\bf P}^{(1)}\sum_{k_1} ({\bf W}_{k_1}^{(1)}\otimes {\bf A}_{k_1}){\bf x}.
\end{split}
\end{equation*}
}

Consider (\ref{Equation: Proof of Lemma MI}) where ${\bf \Sigma}$ is replaced with ${\bf P}^{(l)}\sum_{k_l} ({\bf W}_{k_l}^{(l)}\otimes {\bf A}_{k_l})\cdots {\bf P}^{(2)}\sum_{k_2} ({\bf W}_{k_2}^{(2)}\otimes {\bf A}_{k_2}){\bf P}^{(1)}\sum_{k_1} ({\bf W}_{k_1}^{(1)}\otimes {\bf A}_{k_1})$. 

We deduce the followings:

\begin{proof}[Proof of Theorem~\ref{Theorem: GCNN Information Loss General}]
Suppose $\sigma^{(i)}$ denotes the largest singular value of \  ${\bf P}^{(i)}\sum_{k_{i}=1}^K ({\bf W}_{k_{i}}^{(i)}\otimes {\bf A}_{k_{i}})$ such that $\sigma^{(i)} = \max_j \lambda_j\big({\bf P}^{(i)}\sum_{k_{i}} ({\bf W}_{k_{i}}^{(i)}\otimes {\bf A}_{k_{i}})\big)$. Following the same argument as in the proofs of Theorem~\ref{Theorem: GCN information loss} and~\ref{Theorem: GCN no information loss}, Lemma~\ref{Lemma: MI General} implies that if $\sup_i \sigma^{(i)} <1$, then $\mathcal{I}({\bf x}; {\bf y}^{(l)})=\mathcal{O}\big((\sup_i \sigma^{(i)})^l\big)$, and hence $\lim_{l\rightarrow \infty} \mathcal{I}({\bf x}; {\bf y}^{(l)})=0$ results.
\end{proof}

\begin{proof}[Proof of Theorem~\ref{Theorem: GCNN Information No Loss General}]
We now $\gamma^{(i)}$ denote the minimum singular value of \  ${\bf P}^{(i)}\sum_{k_{i}=1}^K ({\bf W}_{k_{i}}^{(i)}\otimes {\bf A}_{k_{i}})$ such that $\gamma^{(i)} = \min_j \lambda_j\big({\bf P}^{(i)}\sum_{k_{i}=1}^K ({\bf W}_{k_{i}}^{(i)}\otimes {\bf A}_{k_{i}})\big)$. By Lemma~\ref{Lemma: MI General}, it immediately follows that if $\inf_i \gamma^{(i)} \geq 1$, then $\forall l\in \mathbb{N}^+$ we have $\mathcal{L}({\bf y}^{(l)})=0$.
\end{proof}

Before we move on to the proofs of Corollary \ref{Corollary: GCNN information loss decompA} and \ref{Corollary: GCNN no information no loss decompA}, we state the following lemma.

\begin{lemma}\label{Lemma: Decomposition}
Let the singular value decomposition of ${\bf A}\in \mathbb{R}^{n\times n}$ is given by ${\bf A}={\bf U}_{\bf A}{\bf S V}_{\bf A}^T$ and we set each ${\bf A}_k$ to ${\bf A}_k = {\bf U}_{\bf A}{\bf S}_k{\bf V}_{\bf A}^T$ with $({\bf S}_k)_{m, m} = \lambda_m({\bf A})$ if $k=m$ and $({\bf S}_k)_{m, m} = 0$ elsewhere. For such specific composition, we argue that singular values of \ $\sum_k {\bf W}_k\otimes {\bf A}_k$ for ${\bf W}_k \in \mathbb{R}^{d\times d}$ is given by $\lambda_k({\bf A})\lambda_j({\bf W}_k)$ for $k=1, 2, \dots, n$ and $j=1, 2, \dots, d$.
\end{lemma}
\begin{proof}[Proof of Lemma~\ref{Lemma: Decomposition}]
Let the singular value decomposition of ${\bf W}_k$ be ${\bf W}_k={\bf U}_{{\bf W}_k}{\bf S}_{{\bf W}_k}{\bf V}_{{\bf W}_k}^T$. By the property of kronecker product, we have
{\small
\begin{equation*}
    \sum_k {\bf W}_k\otimes {\bf A}_k = \sum_k ({\bf U}_{{\bf W}_k}\otimes {\bf U}_{\bf A})({\bf S}_{{\bf W}_k}\otimes {\bf S}_k)({\bf V}_{{\bf W}_k}^T\otimes {\bf V}_{\bf A}^T).
\end{equation*}
}
Next, we define a set of $nd\times nd$ mask matrices ${\bf M}_k$ such that $({\bf M}_k)_{i, i'}=1$ if $i=i'$ and $i$ (hence $i'$) is of the form $i=k + (j-1)n$ for $j=1, 2, \dots, d$, and $({\bf M}_k)_{i, i'}=0$ otherwise. Reminding that $({\bf S}_k)_{m, m} = \lambda_m({\bf A})$ if $k=m$ and $({\bf S}_k)_{m, m} = 0$ elsewhere, above equation can be rewritten as
{\small
\begin{equation*}
\begin{split}
    \sum_k {\bf W}_k\otimes {\bf A}_k=\sum_k ({\bf U}_{{\bf W}_k}\otimes {\bf U}_{\bf A}){\bf M}_k({\bf S}_{{\bf W}_k}\otimes {\bf S}_k){\bf M}_k({\bf V}_{{\bf W}_k}^T\otimes {\bf V}_{\bf A}^T).
    \end{split}
\end{equation*}
}
In other words, the mask matrix ${\bf M}_k$ applies on the columns (rows) of ${\bf U}_{{\bf W}_k}\otimes {\bf U}_{\bf A}$ $({\bf V}_{{\bf W}_k}^T\otimes {\bf V}_{\bf A}^T)$ where the respective diagonal entries of $({\bf S}_{{\bf W}_k}\otimes {\bf S}_k)$ are nonzero.

Next, we note that if $k=k'$, ${\bf M}_{k}{\bf M}_{k'}={\bf M}_{k}$, and ${\bf M}_{k}$ and ${\bf M}_{k'}$ are orthogonal for $k\neq k'$. This leads us to 
{\small
\begin{equation*}
\begin{split}
    ({\bf U}_{{\bf W}_k}\otimes {\bf U}_{\bf A}){\bf M}_k({\bf S}_{{\bf W}_k}&\otimes {\bf S}_k){\bf M}_k({\bf V}_{{\bf W}_k}^T\otimes {\bf V}_{\bf A}^T)\\
    &=    \sum_{k'} ({\bf U}_{{\bf W}_{k'}}\otimes {\bf U}_{\bf A}){\bf M}_{k'} ({\bf S}_{{\bf W}_{k}}\otimes {\bf S}_k)  \sum_{k''} ({\bf V}_{{\bf W}_{k''}}^T\otimes {\bf V}_{\bf A}^T){\bf M}_{k''}.
\end{split}
\end{equation*}
}
By defining $\tilde{\bf U} =  \sum_k ({\bf U}_{{\bf W}_k}\otimes {\bf U}_{\bf A}){\bf M}_k$ and $\tilde{\bf V} =  \sum_k {\bf M}_k ({\bf V}_{{\bf W}_k}^T\otimes {\bf V}_{\bf A}^T)$ and using the above equation, we resume $\sum_k {\bf W}_k\otimes {\bf A}_k$ as 
\begin{equation}\label{Equation: Lemma 3 main}
    \sum_k {\bf W}_k\otimes {\bf A}_k =  \tilde{\bf U} \sum_k({\bf S}_{{\bf W}_k}\otimes {\bf S}_k) \tilde{\bf V}^T.
\end{equation}

Next, we will show that $\tilde{\bf U}$ and $\tilde{\bf V}$ are unitary matrices through proving that $\tilde{\bf U}\tilde{\bf U}^T=\tilde{\bf U}^T\tilde{\bf U}=\bf{I}$ and $\tilde{\bf V}^T\tilde{\bf V}=\tilde{\bf V}\tilde{\bf V}^T=\bf{I}$. To avoid repeating the same procedure, we will only show it for $\tilde{\bf U}$, but the same result also holds for $\tilde{\bf V}$.

First, we show that {\bf (A.1)} $\tilde{\bf U}\tilde{\bf U}^T=\bf{I}$, and then {\bf (A.2)} $\tilde{\bf U}^T\tilde{\bf U}=\bf{I}$ to argue that $\tilde{\bf U}$ (and $\tilde{\bf V}$) is unitary.

{\bf (A.1)} We can simplify $\tilde{\bf U}\tilde{\bf U}^T$ as
{\small
\begin{equation}\label{Equation: UU^T}
    \begin{split}
       \tilde{\bf U}\tilde{\bf U}^T &= \sum_{k} \big(({\bf U}_{{\bf W}_{k}}\otimes {\bf U}_{\bf A}){\bf M}_{k}\big)  \sum_{k'} \big(({\bf U}_{{\bf W}_{k'}}\otimes {\bf U}_{\bf A}){\bf M}_{k'}\big)^T\\
       \tilde{\bf U}\tilde{\bf U}^T &{=} \sum_{k, k'} \big(({\bf U}_{{\bf W}_{k}}\otimes {\bf U}_{\bf A}){\bf M}_{k}\big) \big(({\bf U}_{{\bf W}_{k'}}\otimes {\bf U}_{\bf A}){\bf M}_{k'}\big)^T\\
       \tilde{\bf U}\tilde{\bf U}^T &\stackrel{(a)}{=} \sum_{k} \big(({\bf U}_{{\bf W}_{k}}\otimes {\bf U}_{\bf A}){\bf M}_{k}\big) \big(({\bf U}_{{\bf W}_{k}}\otimes {\bf U}_{\bf A}){\bf M}_{k}\big)^T
    \end{split}
\end{equation}}
where (a) follows from the orthogonality of ${\bf M}_{k}$ and ${\bf M}_{k'}$ for $k\neq k'$.

We will now take a closer look at $\sum_{k} \big(({\bf U}_{{\bf W}_{k}}\otimes {\bf U}_{\bf A}){\bf M}_{k}\big) \big(({\bf U}_{{\bf W}_{k}}\otimes {\bf U}_{\bf A}){\bf M}_{k}\big)^T$. The entries of summands, $\big(({\bf U}_{{\bf W}_{k}}\otimes {\bf U}_{\bf A}){\bf M}_{k}\big) \big(({\bf U}_{{\bf W}_{k}}\otimes {\bf U}_{\bf A}){\bf M}_{k}\big)^T$, are equivalent to inner product between the rows of $({\bf U}_{{\bf W}_{k}}\otimes {\bf U}_{\bf A}){\bf M}_{k}$ for a fixed $k$. Recall that for a fixed $k$, the mask matrix satisfies $({\bf M}_k)_{i, i}=1$ if $k$ is of the form $i=k+(j-1)n$ for $j=1, 2, \cdots, d$,  and $({\bf M}_k)_{i, i}=0$ elsewhere. We now define $i_\omega$ and $i_\alpha$ as indices such that $i_\omega = \floor{i/n}+1$ and $i_\alpha=\mod(i, \floor{i/n})$. Similarly, let $i'_\omega = \floor{i'/n}+1$ and $i'_\alpha=\mod(i', \floor{i'/n})$. 

Following above definitions, a moment of thought reveals that the nonzero entries of $i$th row of $\big(({\bf U}_{{\bf W}_{k}}\otimes {\bf U}_{\bf A}){\bf M}_{k}\big)$ is given by $({\bf U}_{{\bf W}_k})_{i_\omega, [m=1, 2, \dots, d]}({\bf U}_{\bf A})_{i_\alpha,k}$. We therefore investigate $(\tilde{\bf U}\tilde{\bf U}^T)_{i, i'}$  i.e., the inner product between $i$th and $i'$th rows of $\big(({\bf U}_{{\bf W}_{k}}\otimes {\bf U}_{\bf A}){\bf M}_{k}\big)$ summed over all $k=1, 2, \dots, n$. To start, the inner product between $i$th and $i'$th rows of $\big(({\bf U}_{{\bf W}_{k}}\otimes {\bf U}_{\bf A}){\bf M}_{k}\big)$ is as follows
{\small
\begin{equation}\label{Equation: Inner Product Ukron}
\begin{split}
    \langle& [({\bf U}_{{\bf W}_k})_{i_\omega, [m=1, 2, \dots, d]}({\bf U}_{\bf A})_{i_\alpha,k}], [({\bf U}_{{\bf W}_k})_{i'_\omega, [m=1, 2, \dots, d]}({\bf U}_{\bf A})_{i'_\alpha,k}]\rangle\\
    &=\sum_m ({\bf U}_{{\bf W}_k})_{i_\omega, m}({\bf U}_{\bf A})_{i_\alpha,k}({\bf U}_{{\bf W}_k})_{i'_\omega, m}({\bf U}_{\bf A})_{i'_\alpha,k}\\
    &=\sum_m ({\bf U}_{{\bf W}_k})_{i_\omega, m}({\bf U}_{{\bf W}_k})_{i'_\omega, m}({\bf U}_{\bf A})_{i_\alpha,k}({\bf U}_{\bf A})_{i'_\alpha,k}=({\bf U}_{\bf A})_{i_\alpha,k}({\bf U}_{\bf A})_{i'_\alpha,k}\sum_m ({\bf U}_{{\bf W}_k})_{i_\omega, m}({\bf U}_{{\bf W}_k})_{i'_\omega, m}.
    \end{split}
\end{equation}}
Let now analyze the cases when (1) $i\neq i'$, and (2) $i=i'$.

Assume (1). If further $i_\omega\neq i'_\omega$, it is immediate that $\sum_m ({\bf U}_{{\bf W}_k})_{i_\omega, m}({\bf U}_{{\bf W}_k})_{i'_\omega, m}=0$ by the fact that ${\bf U}_{{\bf W}_k}$ is unitary, hence 
{\small 
\begin{equation*}
    \begin{split}
       \langle [({\bf U}_{{\bf W}_k})_{i_\omega, [m=1, 2, \dots, d]}({\bf U}_{\bf A})_{i_\alpha,k}], [({\bf U}_{{\bf W}_k})_{i'_\omega, [m=1, 2, \dots, d]}({\bf U}_{\bf A})_{i'_\alpha,k}]\rangle=0 
    \end{split}
\end{equation*}}

For (1), if $i_\omega= i'_\omega$, we have $i_\alpha \neq i'_\alpha$. Further, $\sum_m ({\bf U}_{{\bf W}_k})_{i_\omega, m}({\bf U}_{{\bf W}_k})_{i'_\omega, m} = 1$ and hence 
{\small
\begin{equation}
    \begin{split}
        \langle [({\bf U}_{{\bf W}_k})_{i_\omega, [m=1, 2, \dots, d]}&({\bf U}_{\bf A})_{i_\alpha,k}], [({\bf U}_{{\bf W}_k})_{i'_\omega, [m=1, 2, \dots, d]}({\bf U}_{\bf A})_{i'_\alpha,k}]\rangle\\
        &=({\bf U}_{\bf A})_{i_\alpha,k}({\bf U}_{\bf A})_{i'_\alpha,k}\sum_m ({\bf U}_{{\bf W}_k})_{i_\omega, m}({\bf U}_{{\bf W}_k})_{i'_\omega, m}\\
        &=({\bf U}_{\bf A})_{i_\alpha,k}({\bf U}_{\bf A})_{i'_\alpha,k}.
    \end{split}
\end{equation}
}
Hence, the inner product between $i$th and $i'$th rows of $\big(({\bf U}_{{\bf W}_{k}}\otimes {\bf U}_{\bf A}){\bf M}_{k}\big)$ is given by $({\bf U}_{\bf A})_{i_\alpha,k}({\bf U}_{\bf A})_{i'_\alpha,k}$. Recalling (\ref{Equation: UU^T}), we have $(\tilde{\bf U}\tilde{\bf U}^T)_{i,i'} = \sum_{k} ({\bf U}_{\bf A})_{i_\alpha,k}({\bf U}_{\bf A})_{i'_\alpha,k}$. As previously mentioned we have $i_\alpha \neq i'_\alpha$. By the unitary property of ${\bf U}_{\bf A}$, we further have $(\tilde{\bf U}\tilde{\bf U}^T)_{i,i'} = \sum_{k} ({\bf U}_{\bf A})_{i_\alpha,k}({\bf U}_{\bf A})_{i'_\alpha,k}=0$.

So far we have shown that $(\tilde{\bf U}\tilde{\bf U}^T)_{i, i'}=0$ when $i\neq i'$. Let now $i=i'$, i.e., (2). IT follows from (\ref{Equation: Inner Product Ukron}) that
{\small
\begin{equation}\label{Equation: UU^T same index}
    \begin{split}
        (\tilde{\bf U}\tilde{\bf U}^T)_{i, i} \stackrel{(a)}{=} \sum_k ({\bf U}_{\bf A})^2_{i_\alpha,k}\sum_m({\bf U}_{{\bf W}_k})^2_{i_\omega, m}        (\tilde{\bf U}\tilde{\bf U}^T)_{i, i} \stackrel{(b)}{=} \sum_k ({\bf U}_{\bf A})^2_{i_\alpha,k}1
        (\tilde{\bf U}\tilde{\bf U}^T)_{i, i} \stackrel{(c)}{=} 1
    \end{split}
\end{equation}
}
where (a) results from that ${\bf U}_{{\bf W}_k}$ is unitary, and (b) follows from that ${\bf U}_{{\bf A}}$ is unitary. Combining~above arguments and (\ref{Equation: UU^T same index}), we have $\tilde{\bf U}\tilde{\bf U}^T=\bf{I}$.

{\bf (A.2)} Next, we show that $\tilde{\bf U}^T\tilde{\bf U}=\bf{I}$. We begin with
{\small
\begin{equation}
\label{Equation: U^TU main}
    \begin{split}
        \tilde{\bf U}^T\tilde{\bf U} &= \sum_{k} \big(({\bf U}_{{\bf W}_{k}}\otimes {\bf U}_{\bf A}){\bf M}_{k}\big)^T  \big(\sum_{k'} ({\bf U}_{{\bf W}_{k'}}\otimes {\bf U}_{\bf A}){\bf M}_{k'}\big)\tilde{\bf U}^T\tilde{\bf U}\\
        &=\sum_{k, k'} \big(({\bf U}_{{\bf W}_{k}}\otimes {\bf U}_{\bf A}){\bf M}_{k}\big)^T  \big(({\bf U}_{{\bf W}_{k'}}\otimes {\bf U}_{\bf A}){\bf M}_{k'}\big).
    \end{split}
\end{equation}
}
For $k\neq k'$,
{\small
\begin{equation}
    \begin{split}
    \Big( \big(({\bf U}_{{\bf W}_{k}}\otimes {\bf U}_{\bf A}){\bf M}_{k}\big)^T  \big(({\bf U}_{{\bf W}_{k'}}\otimes {\bf U}_{\bf A}){\bf M}_{k'}\big)\Big)_{i, i'}=\langle \big(({\bf U}_{{\bf W}_{k}}\otimes {\bf U}_{\bf A}){\bf M}_{k}\big)_i, \big(({\bf U}_{{\bf W}_{k'}}\otimes {\bf U}_{\bf A}){\bf M}_{k'}\big)_{i'}\rangle.
   \end{split}
\end{equation}
}

Note that, due to the orthogonality of ${\bf M}_{k}$ and ${\bf M}_{k}$ for $k\neq k'$, we further have $\langle \big(({\bf U}_{{\bf W}_{k}}\otimes {\bf U}_{\bf A}){\bf M}_{k}\big)_i, \big(({\bf U}_{{\bf W}_{k'}}\otimes {\bf U}_{\bf A}){\bf M}_{k'}\big)_{i'}\rangle=0$ for $i\neq i'$. When $i=i'$, on the other hand, we have

{
\begin{equation}
    \begin{split}
    & \Big( \big(({\bf U}_{{\bf W}_{k}}\otimes {\bf U}_{\bf A}){\bf M}_{k}\big)^T  \big(({\bf U}_{{\bf W}_{k'}}\otimes {\bf U}_{\bf A}){\bf M}_{k'}\big)\Big)_{i, i'}\\ 
    &=\langle \big(({\bf U}_{{\bf W}_{k}}\otimes {\bf U}_{\bf A}){\bf M}_{k}\big)_i, \big(({\bf U}_{{\bf W}_{k'}}\otimes {\bf U}_{\bf A}){\bf M}_{k'}\big)_i\rangle\\
    &\stackrel{(a)}{=} \langle ({\bf U}_{{\bf W}_k})_{[z=1, \cdots, d], i_\omega}({\bf U}_{\bf A})_{[w=1, \cdots, n], k}, ({\bf U}_{{\bf W}_{k'}})_{[z=1, \cdots, d], i_\omega}({\bf U}_{\bf A})_{[w=1, \cdots, n], k'}\rangle\\
     &= \sum_{w} \sum_d ({\bf U}_{{\bf W}_k})_{z, i_\omega}({\bf U}_{\bf A})_{w, k}({\bf U}_{{\bf W}_{k'}})_{z, i_\omega}({\bf U}_{\bf A})_{w, k'}\\
     &\stackrel{(b)}{=} \sum_d ({\bf U}_{{\bf W}_k})_{z, i_\omega}({\bf U}_{{\bf W}_{k'}})_{z, i_\omega}\sum_w ({\bf U}_{\bf A})_{w, k}({\bf U}_{\bf A})_{w, k'}\\
     &= 0
   \end{split}
\end{equation}
}
where (a) follows from that $ \big(({\bf U}_{{\bf W}_{k}}\otimes {\bf U}_{\bf A}){\bf M}_{k}\big)_i=({\bf U}_{{\bf W}_k})_{[z=1, \cdots, d], i_\omega}({\bf U}_{\bf A})_{[w=1, \cdots, n], k}$ and (b) results from that $\sum_w ({\bf U}_{\bf A})_{w, k}({\bf U}_{\bf A})_{w, k'}=0$ for $k\neq k'$ as ${\bf U}_{\bf A}$ is unitary.

Therefore, (\ref{Equation: U^TU main}) can be resumed as
{\small
\begin{equation*}
    \begin{split}
        \tilde{\bf U}^T\tilde{\bf U}&=\sum_{k} \big(({\bf U}_{{\bf W}_{k}}\otimes {\bf U}_{\bf A}){\bf M}_{k}\big)^T  \big(({\bf U}_{{\bf W}_{k}}\otimes {\bf U}_{\bf A}){\bf M}_{k}\big)\\
        \tilde{\bf U}^T\tilde{\bf U}&= \sum_k{\bf M}_k({\bf U}_{{\bf W}_{k}}\otimes {\bf U}_{\bf A})^T ({\bf U}_{{\bf W}_{k}}\otimes {\bf U}_{\bf A}) {\bf M}_k\\
        \tilde{\bf U}^T\tilde{\bf U}&\stackrel{(a)}{=} \sum_k {\bf M}_k\bf{I} {\bf M}_k=\sum_k {\bf M}_k \stackrel{(b)}{=} \bf{I}
    \end{split}
\end{equation*}}
where (a) follows from that the kronecker product of unitary matrices is also unitary, hence $({\bf U}_{{\bf W}_k}\otimes {\bf U}_{\bf A})$ is unitary, and (b) follows from the definition of ${\bf M}_k$.

As the last step, recall from (\ref{Equation: Lemma 3 main}) that $\sum_k {\bf W}_k\otimes {\bf A}_k =  \tilde{\bf U} \sum_k({\bf S}_{{\bf W}_k}\otimes {\bf S}_k) \tilde{\bf V}^T$, and note by the definition of ${\bf S}_k$ that $({\bf S}_{{\bf W}_k}\otimes {\bf S}_k)_{i, i'}=\lambda_k({\bf A})\lambda_j({\bf S}_{{\bf W}_k})$ if $i=i'$ and $i$, hence $i'$, of the form $i=k+(j-1)n$ for $j=1, 2, \cdots, d$, and $({\bf S}_{{\bf W}_k}\otimes {\bf S}_k)_{i, i'}=0$ elsewhere. Therefore, by the fact that $({\bf S}_{{\bf W}_k}\otimes {\bf S}_k)({\bf S}_{{\bf W}_{k'}}\otimes {\bf S}_{k'})=0$ for $k\neq k'$, it follows that $\sum_k({\bf S}_{{\bf W}_k}\otimes {\bf S}_k)$ is a diagonal matrix with diagonal entries $\lambda_k({\bf A})\lambda_j({\bf S}_{{\bf W}_k})$ where $j=1, 2, \cdots, d$ and $k=1, 2, \cdots, n$, which completes the proof.
\end{proof}


For the decomposition of ${\bf A}$ such that ${\bf A}_k={\bf U}_{\bf A}{\bf S}_k{\bf V}_{\bf A}^T$ where the singular value decomposition of ${\bf A}$ is given by ${\bf A}={\bf U}_{\bf A}{\bf S V}_{\bf A}^T$, we recall Theorem~\ref{Theorem: GCNN Information Loss General} and \ref{Theorem: GCNN Information No Loss General} to conclude Corollary~\ref{Corollary: GCNN information loss decompA} and \ref{Corollary: GCNN no information no loss decompA} as follows.

\begin{sloppypar}
\begin{proof}[Proof of Corollary~\ref{Corollary: GCNN information loss decompA}]
Let $\sigma_{{\bf A}_k} = \lambda_k({\bf A})$ and $\sigma_{{\bf W}_k}=\sup_i\max_j\lambda_j({\bf W}_k^{(i)})$. By Lemma~\ref{Lemma: Decomposition}, we have $\max_j \lambda_j(\sum_k ({\bf W}^{(i)}_k\otimes {\bf A}_k))\leq \max_k \sigma_{{\bf A}_k}\sigma_{{\bf W}_k}$. Noting that ${\bf P}^{(i)}$ is diagonal with entries at most 1, we have $\max_j \lambda_j\big({\bf P}^{(l)}\sum_{k_l} ({\bf W}_{k_l}^{(l)}\otimes {\bf A}_{k_l})\cdots {\bf P}^{(2)}\sum_{k_2} ({\bf W}_{k_2}^{(2)}\otimes {\bf A}_{k_2}){\bf P}^{(1)}\sum_{k_1} ({\bf W}_{k_1}^{(1)}\otimes {\bf A}_{k_1})\big)\leq (\max_k \sigma_{{\bf A}_k}\sigma_{{\bf W}_k})^l$. Therefore, if $\underline{\forall k=\{1, 2, \dots, n\}}$ $\sigma_{{\bf A}_k}\sigma_{{\bf W}_k}<1$, then $\lim_{l\rightarrow \infty}\max_j\lambda_j\big(\sum_k ({\bf W}^{(i)}_k\otimes {\bf A}_k)\big)=0$. Hence $\lim_{l \rightarrow \infty} \mathcal{I}({\bf x}; {\bf y}^{(l)})=0$ results by Lemma~\ref{Lemma: MI General}.
\end{proof}
\end{sloppypar}

\begin{sloppypar}
\begin{proof}[Proof of Corollary~\ref{Corollary: GCNN no information no loss decompA}]
Let $\gamma_{{\bf W}_k} = \inf_i \min_j \lambda_j({\bf W}_k^{(i)})$. Note that $\min_j \lambda_j \big( {\bf P}^{(i)}\sum_k {\bf W}_k^{(i)}\otimes {\bf A}_k\big)\geq a\min_k \lambda_k({\bf A})\gamma_{{\bf W}_k}$ by Lemma~\ref{Lemma: Decomposition} and that $\min_j \lambda_j({\bf P}^{i})=a$. Moreover, $\min_j \lambda_j\big({\bf P}^{(l)}\sum_{k_l} ({\bf W}_{k_l}^{(l)}\otimes {\bf A}_{k_l})\cdots {\bf P}^{(2)}\sum_{k_2} ({\bf W}_{k_2}^{(2)}\otimes {\bf A}_{k_2}){\bf P}^{(1)}\sum_{k_1} ({\bf W}_{k_1}^{(1)}\otimes {\bf A}_{k_1})\big) \geq ( a\min_k \lambda_k({\bf A})\gamma_{{\bf W}_k})^l$. Therefore, if $a\sigma_{{\bf A}_k}\gamma_{{\bf W}_k}\geq1$, $\underline{\forall k \in \{1, 2, \dots, n\}}$, then $\mathcal{I}({\bf x}; {\bf y}^{(l)})=\mathcal{H}({\bf x})$ $\forall l\in \mathbb{N}^+$ by Lemma~\ref{Lemma: MI General}, hence $\mathcal{L}({\bf y}^{(l)}) = 0$.
\end{proof}
\end{sloppypar}

\section{Datasets description} 
\label{apdx:dataset}
Cora, Citeseer, and Pubmed\footnote{https://github.com/tkipf/gcn/tree/master/gcn/data} are three well-known citation network datasets, and we follows the same training/validation/test split as GCN~\cite{DBLP:conf/iclr/KipfW17}.
Reddit is a social network dataset modeling the community structure of Reddit posts. This dataset is often used as an inductive training setting and the training/validation/test split is coherent with that of GraphSAGE~\cite{hamilton2017inductive}.
Flickr originates from NUS-wide~\footnote{http://lms.comp.nus.edu.sg/research/NUS-WIDE.html} and contains different types of images based on the descriptions and common properties of online images. We use a public version of Reddit and Flickr provided by GraphSAINT\footnote{https://github.com/GraphSAINT/GraphSAINT}.

Amazon Computers and Amazon Photo are segments of the Amazon co-purchase graph~\cite{DBLP:conf/sigir/McAuleyTSH15}, where nodes represent goods, edges indicate that two goods are frequently bought together, node features are bag-of-words encoded product reviews, and class labels are given by the product category.
Coauthor CS and Coauthor Physics are co-authorship graphs based on the Microsoft Academic Graph from the KDD Cup 2016 challenge\footnote{https://kddcup2016.azurewebsites.net/}. Here, nodes are authors, that are connected by an edge if they co-authored a paper; node features represent paper keywords for each author’s papers, and class labels indicate most active fields of study for each author. We use a pre-divided version of these datasets through the Deep Graph Library (DGL)\footnote{https://docs.dgl.ai/en/0.4.x/api/python/data.html\#coauthor-dataset}. Actor is an actor-only induced subgraph of the film-director-actor-writer network~\cite{DBLP:conf/kdd/TangSWY09}. Each nodes correspond to an actor, and the edge between two nodes denotes co-occurrence on the same Wikipedia page. Node features correspond to some keywords in the Wikipedia pages.\footnote{https://github.com/graphdml-uiuc-jlu/geom-gcn/tree/master/new\_data/film}

The company dataset is a real-world transaction graph which we used for fraud transactions detection. Historical transaction records spanning a given period of time were extracted for graph construction. We treat each transaction as a node and assume there is an edge between two nodes if they have the same hard linkage, such as purchasing by the same buyer, shipping to the same address or using the same financial instruments etc. Node features are constructed from individual risk factors. To reduce graph size and meanwhile preserve graph connectivity, we adopt a graph sampling strategy: firstly, all fraudulent transactions and random sampled normal transactions are selected as seeds; secondly, each seed is expanded to its 3-hop neighbors, at each hop, no more than 32 neighbors are picked. Thirdly, those groups with transaction numbers less than 5 are filtered out. There are two different sizes of transaction graph (company dataset small and large) vary in transaction spanning periods and number of individual features.

\section{Experiments setup}

\subsection{Experiments setup for the experiments on arbitrary graph-structured data}
\label{apdx:setup}
The hyperparameters (e.g., learning rate, number of hidden units) are selected from grid search. The grid search was performed over the following search space:

\begin{itemize}
    \item Hidden size: \texttt{[8, 16, 32, 64, 128, 256, 512]}
    \item Learning rate: \texttt{[0.001, 0.003, 0.005, 0.008, 0.01]}
    \item Partition numbers $K$: \texttt{[2,3,4,5,6,7,8]}
    \item Parameter $p$ in METIS: \texttt{[40,80,100,150,180,200,250,500,1000]}
    \item Dropout probability: \texttt{[0.2, 0.3, 0.4, 0.5, 0.6, 0.7, 0.8, 0.85, 0.9]}
    \item $L_2$ regularization strength: \texttt{[1e-4, 5e-4, 1e-3, 5e-3, 1e-2, 5e-2, 1e-1]}
    \item Attention coefficients dropout probability (only for GAT):
    \texttt{[0.2, 0.3, 0.4, 0.5, 0.6, 0.7, 0.8]}
\end{itemize}

\paragraph{Random decomposition settings}
In this experiment, we evaluated on a standard two-layer GCN on Cora. All layers are randomly decomposed with the same parameter $K$. In particular, the random decomposition distributes all the edges to the $K$ subgraphs in a round-robin manner and there are no common edges among these subgraphs. Other hyperparameters are selected from grid search.

\paragraph{Results in Table 1}
In the table, $^*$ indicates that we ran our own implementation. JK-Net, ResGCN, DenseGCN and the DeGNN version in our implementation are built on top of standard GCN on both transductive and inductive settings. Most of the other testing accuracy are directly collected from the corresponding original paper. Except for 1) DropEdge on citation datasets reused from the openreview results\footnote{https://openreview.net/forum?id=Hkx1qkrKPr}, 2) GraphSAGE on Flickr reused from GraphSAINT~\cite{DBLP:conf/iclr/ZengZSKP20}. On the inductive tasks, DropEdge was not evaluated on Flickr in the original paper, thereby, we also ran an experiment with the codes\footnote{https://github.com/DropEdge/DropEdge}. Note that, we cannot reproduce the experimental results for DropEdge+JKNet on Reddit, as reported in their paper, according to its setting. To make it fair, all DropEdge results in our experiments indicate using GCN as the backbones by default.

\paragraph{DeGNN Implementation Settings} 
We use PyTorch to implement the models and we train them using Adam optimizer. Besides, we train each model 400 epochs and terminate the training process if the validation accuracy does not improve for 20 consecutive steps. Note that JK-Net has three aggregators, and we choose the concatenation as the final aggregation layer since it performs best in most cases. Every experiment is ran ten times and the mean accuracy is reported.
For inductive tasks, the training procedure is on the training set. The validation set and testing set are added into the graph only for the prediction. Therefore, we need not only perform decomposition on the training graph, but also continue to decompose the whole graph for new nodes and edges. The second decomposition reuses the same $K$ decomposition and the trained DeGNN model to make predictions. We notice that DropEdge utilizes a self feature modeling~\cite{DBLP:conf/nips/FoutBSB17} operation on GCN in their implementation, which is fundamental to the reported testing accuracy. To make it fair, we only involve this trick for our implementations in the experiment on Reddit.
To get the best hyper-parameters on different datasets, we adopts grid-search for each model on and report the case who has the best validation accuracy in Table~\ref{tab:hyper}.

\begin{table}[t]
\small
\centering
\caption{Overview of the Four Datasets} \label{tab:hyper}
\begin{tabular}{cccc}
\toprule
\textbf{Dataset}&\textbf{Model}& \textbf{Accuracy}&\textbf{Hyper-parameters}\\
\midrule
{Cora} & DeGNN(Dense) & 84.3 &lr:0.01, hidden size:128, nlayers:5,\\& & & $K$: [3,3,2,2], metis:40, dropout:0.9, weight-decay:5e-4 \\
\midrule
Citeseer& DeGNN(JK) & 73.1&lr:0.01, hidden size:64, nlayers:4,\\& & & $K$: [3,2,2], metis:100, dropout:0.9, weight-decay:5e-4\\
\midrule
Pubmed& DeGNN(Dense) & 80.1 &lr:0.01, hidden size:128, nlayers:4,\\& & & $K$: [3,3,2], metis:100, dropout:0.85, weight-decay:5e-4\\
\midrule
Flickr& DeGNN(JK) & 52.5 & lr:0.008, hidden size:128, nlayers:4,\\& & & $K$: [3,3,2], metis:100, dropout:0.5, weight-decay:5e-4 \\
\midrule
Reddit& DeGNN(Dense) & 96.7 & lr:0.01, hidden size:256, nlayers:3,\\& & & $K$: [3,2], metis:100, dropout:0.2, weight-decay:5e-4,\\& & & with self feature modeling  \\
\bottomrule
\end{tabular}
\end{table}

\section{Further Experiments}
\subsection{Model Analysis on the graph connectivity}
\label{apdx:connect}
\begin{wrapfigure}{r}{0.45\textwidth}
\centering
\includegraphics[width=1.0\linewidth]{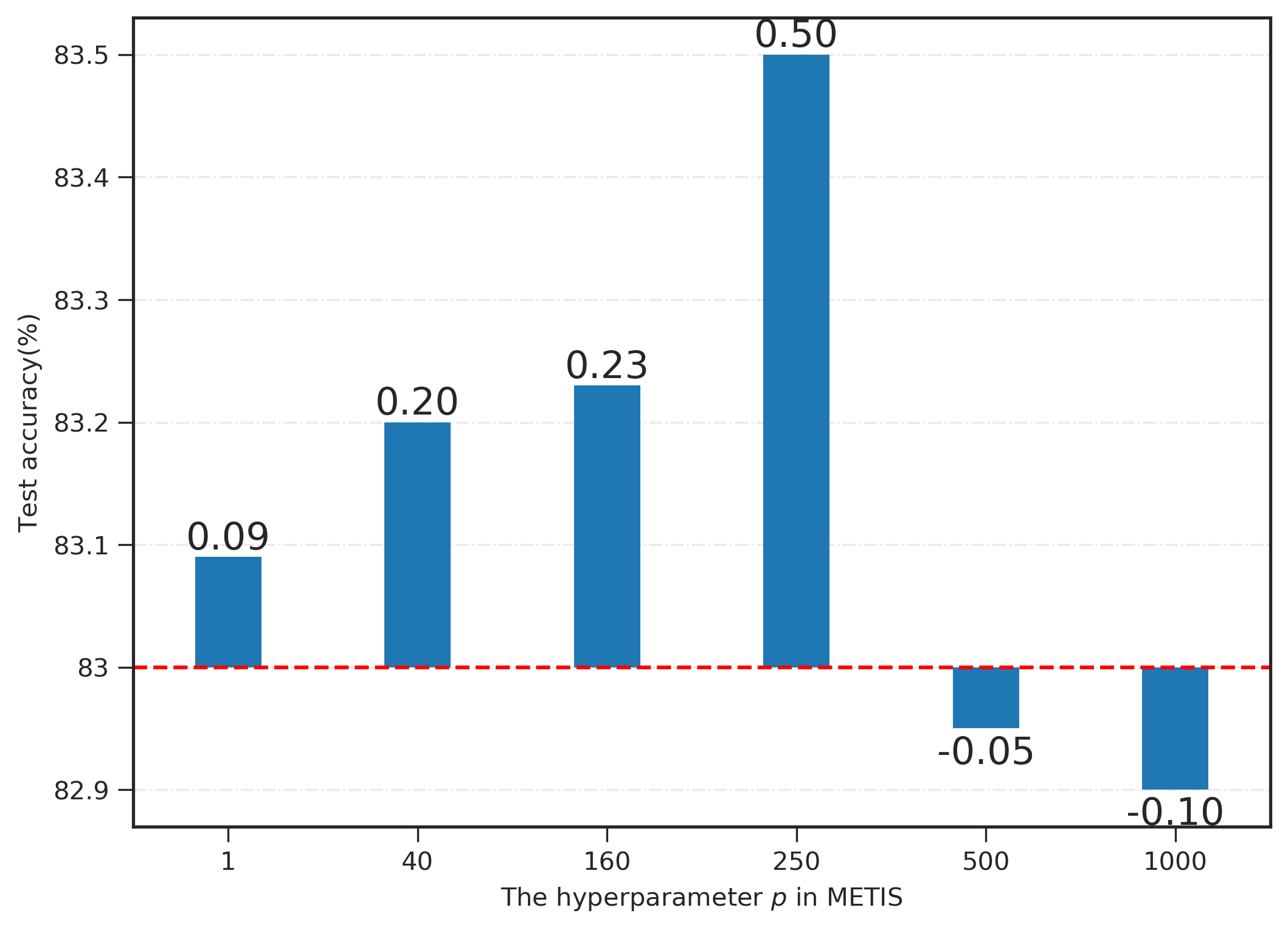}
\caption{The testing accuracy (\%) on Cora with different $p$.}\label{fig:metis}
\vspace{-1em}
\end{wrapfigure}
The proposed spanning forest based graph decomposition can control the graph connectivity with the METIS partition step. 
For simplicity, we evaluate on a standard two-layer GCN on Cora, and only the first layer is replace by a decomposed GCN with $K=4$. Other hyperparameters are selected from the grid search. We tune the parameter $p$ in METIS and get different connected components to generate different sizes of spanning tree. Then we test its influence on the final testing accuracy. 
As shown in Figure~\ref{fig:metis}, as $p$ increases, the testing accuracy improves at first, but drops down quickly at last. This is because the METIS eliminate too many edges cuts and result in a loss of graph connectivity.
In the above experiments, we can confirm that the graph decomposition do contribute to better performance to some extent, and proper connectivity is also significant to achieve great performance. There exists a trade-off between graph decomposition and graph connectivity.

\subsection{Model Analysis on the mode depth}
\label{apdx:depth}
Table~\ref{tab:depth} shows a detailed version of the influence of model depth for different models on the three citation datasets.

\begin{table}[t]
\small
\centering
\caption{Testing accuracy (\%) comparisons on different models w and w/o DeGNN} \label{tab:depth}
\begin{tabular}{cc|cc|cc|cc}
\toprule
 & &\multicolumn{2}{c|}{\textbf{4 layers}}&\multicolumn{2}{c|}{\textbf{6 layers}}&\multicolumn{2}{c}{\textbf{8 layers}}\\
\textbf{Dataset}&\textbf{Model} & \textbf{Original} & \textbf{DeGNN} & \textbf{Original} & \textbf{DeGNN} & \textbf{Original} & \textbf{DeGNN}\\
\midrule
\multirow{4}{*}{Cora} 
& GCN  & 80.2 & \textbf{82.8} & 74.3 & \textbf{80.5} & 59.4 & \textbf{75.4}\\
& ResGCN  & 81.2 &\textbf{83.7} & 80.7 & \textbf{83.4} & 80.5 & \textbf{82.4}\\
& JK-Net  &81.2 & \textbf{83.6} & 81.8 & \textbf{\underline{83.9}} & 81.6 & \textbf{\underline{83.7}}\\
& DenseGCN  & 82.1 & \textbf{\underline{84.0}} & 81.5 & \textbf{83.5} & 81.3 & \textbf{83.3}\\
\midrule
\multirow{4}{*}{Citeseer} 
& GCN  & 63.8 & \textbf{72.3} & 62.2 & \textbf{70.3} & 47.4 & \textbf{64.4}\\
& ResGCN  & 70.1 & \textbf{72.4} & 70.0 & \textbf{71.8} & 69.6 & \textbf{71.8}\\
& JK-Net  &70.5 & \textbf{\underline{73.1}} & 70.3 & \textbf{\underline{72.8}} & 70.6 & \textbf{72.7}\\
& DenseGCN  & 71.1 & \textbf{72.5} & 70.7 & \textbf{72.5} & 70.6 & \textbf{\underline{72.8}}\\
\midrule
\multirow{4}{*}{Pubmed} 
& GCN  & 74.4 & \textbf{79.1} & 72.7 & \textbf{77.4} & 68.1 & \textbf{76.1}\\
& ResGCN  & 78.3 & \textbf{79.5} & 78.0 & \textbf{79.6} & 77.9 & \textbf{79.4}\\
& JK-Net  &78.8 & \textbf{80.0} & 78.6 & \textbf{\underline{79.8}} & 78.5 & \textbf{\underline{79.6}}\\
& DenseGCN  & 78.9 & \textbf{\underline{80.1}} & 79.0 & \textbf{79.4} & 79.0 & \textbf{79.2}\\
\bottomrule
\end{tabular}
\end{table}

\section{Complexity Analysis}
The time complexity of original GCN (comes from the sparse-dense matrix multiplications) is $\mathcal{O}(LMF+LNF^2)$, where $L$ is the number of layers, $N$ is the number of nodes, $F$ is the number of features and $M$ is the number of edges, i.e., nonzero elements in adjacency matrix $A$. The graph decomposition consists of the METIS step ($\mathcal{O}(N+M+p\log(p))$), the spanning forest generation (e.g., DFS $\mathcal{O}(N+M)$) and the node decomposition strategies ($\mathcal{O}(M)$). Therefore, the major computation of DeGNN is nearly the same with GCN asymptotically.

\end{document}